\documentclass[twoside,11pt]{article}
\usepackage[abbrvbib, preprint]{jmlr2e}
\usepackage{jmlr2e}
\usepackage{amsmath}
\usepackage{enumerate}
\usepackage{graphicx}
\usepackage{subfigure}
\usepackage{amssymb}
\usepackage{setspace}
\usepackage{mathrsfs}
\usepackage{color}
\usepackage{multirow} 
\usepackage{array}
\usepackage{pdfpages}
\usepackage{parskip}
\usepackage{tabularx}
\usepackage{arydshln}
\usepackage{booktabs}
\usepackage{threeparttable}
\usepackage{placeins}
\usepackage{makecell}
\usepackage[cmyk]{xcolor}

\usepackage{lastpage}
\jmlrheading{26}{2025}{1-\pageref{LastPage}}{x/xx; Revised x/xx}{x/xx}{xx-xxxx}{Xiao-Song Yang, Qi Zhou and Xuan Zhou }

\ShortHeadings{Minimum Width of Deep Narrow Networks}{Yang, Zhou and Zhou}
\firstpageno{1}

\begin{document}

\title{Minimum Width of Deep Narrow Networks for Universal Approximation}
\author{\name Xiao-Song Yang \email yangxs@hust.edu.cn \\
       \addr School of Mathematics and Statistics\\
       Huazhong University of Science and Technology\\
       Wuhan 430074, P.R. China\\
       Hubei Key Laboratory of Engineering Modeling and Scientifc Computing\\
       Huazhong University of Science and Technology\\
       Wuhan 430074, P.R. China
       \AND
       \name Qi Zhou \email qizhou1037@hust.edu.cn\\
       \addr School of Mathematics and Statistics\\
       Huazhong University of Science and Technology\\
       Wuhan 430074, P.R. China
       \AND
       \name Xuan Zhou \email xuanzhou1037@hust.edu.cn\\
       \addr School of Mathematics and Statistics\\
       Huazhong University of Science and Technology\\
       Wuhan 430074, P.R. China
       }

\editor{My editor}

\maketitle

\begin{abstract}
	Determining the minimum width of fully connected neural networks has become a fundamental problem in recent theoretical studies of deep neural networks. In this paper, we study the lower bounds and upper bounds of the minimum width required for fully connected neural networks in order to have universal approximation capability, which is important in network design and training. We show that $w_{min}\leq\max(2d_x+1, d_y)$ also holds true for networks with ELU, SELU activation functions, and the upper bound of this inequality is attained when $d_y=2d_x$, where $d_x$, $d_y$ denote the input and output dimensions, respectively. Besides, we show that $d_x+1\leq w_{min}\leq d_x+d_y$ for networks with LeakyReLU, ELU, CELU, SELU, Softplus activation functions, by proving that ReLU activation function can be approximated by these activation functions. In addition, in the case that the activation function is injective or can be uniformly approximated by a sequence of injective functions (e.g., ReLU), we present a new proof of the inequality $w_{min}\ge d_y+\mathbf{1}_{d_x<d_y\leq2d_x}$ by constructing a more intuitive example via a new geometric approach based on Poincar\'e-Miranda Theorem. 
\end{abstract}

\section{Introduction}\label{sec_1}
It is well known now that the architecture of deep neural networks (DNNs), particularly the width and depth of DNNs, significantly influences the performance of DNNs. Since empirical observations have shown that width has a significant impact on the performance of neural networks, a lot of research have been dedicated to the effect of width on the expressive power of neural networks in terms of approximating various functions as well as expected. In this line of research, deep narrow neural networks, which refer to the networks with large depth but extremely small width, have been an active research area of DNNs. 

Mathematically, the expressive power of neural networks is characterized by the so called Universal Approximation Property (UAP), which refers to the ability of neural networks to approximate a wide range of functions. Recent studies have shown that, to guarantee the UAP, the minimum network width of DNNs depends not only on the input and output dimensions, but also on the choice of activation function. 
\subsection{Why Minimum Width is Important}\label{sec_1.1}
Recent studies have shown the advantages of reducing network width and the disadvantages of excessively large network width in DNNs training and their generalization. Based on the Neural Tangent Kernel (NTK) theory, \cite{NTK} proved that networks with smaller width appear to converge faster than wider ones. Besides, it is shown that narrow networks require fewer parameters for UAP (\cite{Yarotsky}) and memorization (\cite{Park_memorization}; \cite{Vardi}) than shallow networks. Besides, several works have shown that narrow neural networks may exhibit better generalization ability (\cite{ICLR2025_narrowbetter}), and that once a DNN achieves the sufficient capacity, width is harmful to its fitting and performance (\cite{limitation_of_large_width}). To reduce training time and memory costs while improving the ability of generalization, a key problem along this line of research is to reduce the network width as much as possible. 

However, several studies have shown that the network width can not be arbitrarily reduced while preserving the approximation ability. \cite{Johnson} proved that networks with injective continuous activation functions and network width less than $d_x+1$ can not approximate arbitrary function. \cite{Cai} proved that for an arbitrary activation function, network width must exceed or equal to $\max(d_x, d_y)$ to achieve the UAP. \cite{Kim} proved that network width must exceed or equal to $d_y+\mathbf{1}_{d_x<d_y\leq2d_x}$ for injective activation functions. Therefore, the minimum width $w_{min}$, which is defined as the minimum sufficient hidden layer width that guarantees the UAP in the function space $C(\mathbb{R}^m, \mathbb{R}^n)$, is significant for the DNN design and training.

Consequently, the minimum width is of both theoretical and practical significance. It provides a guidance for training a neural network while minimizing the consumption of computational resources. 
\subsection{Our Contributions}\label{sec_1.2}
In this paper, we investigate the minimum width for lots of activation functions. For reader's convenience, the state-of-the-art (SOTA) results on the minimum width $w_{min}$ with our contributions, are summarized in Table \ref{table_summary_results}. In this table, an activation function is referred to as uniformly continuous if it is injective or can be uniformly approximated by a sequence of injective functions, and the variants of ReLU are LeakyReLU, ELU, CELU, SELU, Softplus.

Our main contributions can be summarized into the following several parts.
\begin{itemize}
	\item We show the inequality of minimum width that $w_{min}\leq \max(2d_x+1, d_y)$ (Theorem \ref{thm_variant_LR_width}) for networks with ELU and SELU. We further show that the upper bound $w_{min}=\max(2d_x+1, d_y)$ is attained when $d_y=2d_x$ (Theorem \ref{thm_LRwidth}). These results extend the range of activation functions in (\cite{Hwang}).
	\item We show that $d_x +1\leq w_{min}\leq d_x+d_y$ (Theorem \ref{thm_variant_ReLU_width}) for networks with variants of ReLU (LeakyReLU, ELU, CELU, SELU, Softplus), which extends the range of activation functions in (\cite{Hanin}).
	\item[*] In addition, we construct a more intuitive example via a new geometric approach based on the Poincar\'e-Miranda Theorem, thus giving a new proof of the inequality $w_{min}\geq d_y+\mathbf{1}_{d_x<d_y\leq 2d_x}$ obtained by \cite{Kim}, for the case that the activation function is injective or uniformly approximated by a sequence of injective functions.
\end{itemize}

\begin{table}[htbp]
	\small
	\centering
	\caption{Summary of SOTA results on the minimum width for different activation functions. $K$ denotes a compact domain with dimention $d_x$, and $d_y$ denotes the dimension of the range space.}
	\label{table_summary_results}
	\begin{threeparttable}  
		\scalebox{0.9}{
		\begin{tabular}{|c | c | c | c|}
			\hline
			Reference & Domain & Activation function $\sigma$ & Upper/lower bounds \\
			\hline
			\hline
			\cite{Hanin} & $C(K, \mathbb{R}^{d_y})$ & ReLU & $d_x+1 \leq w_{min}\leq d_x+d_y$\\
			\hline
			\cite{Johnson}& $C(K, \mathbb{R})$& uniformly continuous\tnote{1} & $w_{min}\geq d_x+1$\\
			\hline
			\multirow{2}{*}{\cite{Kidger}}& $C(K, \mathbb{R}^{d_y})$& continuous nonpoly & $w_{min}\leq d_x+d_y+1$\\
			\cline{2-4}
			& $C(K, \mathbb{R}^{d_y})$& nonaffine poly & $w_{min}\leq d_x+d_y+2$\\
			\hline
			\cite{Park} & $C(K, \mathbb{R}^{d_y})$ & ReLU+STEP & $w_{min}=\max(d_x+1, d_y)$\\
			\hline
			\multirow{1}{*}{\cite{Cai}}& $C(K, \mathbb{R}^{d_y})$& Arbitrary & $w_{min}\geq \max(d_x, d_y)$\\
			\hline
			\multirow{6}{*}{\cite{Hwang}} & $C(K, \mathbb{R}^{d_y})$ & LeakyReLU & $w_{min}\leq \max(2d_x+1, d_y)$\\
			\cline{2-4}
			& $C(K, \mathbb{R}^{d_y})$ & ReLU & $w_{min}\leq \max (2d_x+1, d_y)+1$\\
			\cline{2-4}
			& $C(K, \mathbb{R}^{d_y})$ & continuous nonpoly & $w_{\text{min}}\leq \max(2d_x+1, d_y)+2$\\
			\cline{2-4}
			& $C([0, 1]^2, \mathbb{R}^2)$ & ReLU & $w_{min}=4$\\
			\cline{2-4}
			& $C([0, 1]^2, \mathbb{R}^2)$ & LeakyReLU & $w_{min}=4$\\
			\cline{2-4}
			& $C([0, 1]^2, \mathbb{R}^2)$ & uniformly continuous\tnote{1} & $w_{min}\geq 4$\\
			\hline
			\cite{Kim} & $C(K, \mathbb{R}^{d_y})$ & uniformly continuous\tnote{1} & $w_{min}\geq d_y+\mathbf{1}_{d_x<d_y\leq2d_x}$\\
			\hline
			\textbf{Ours} (Theorem \ref{thm_variant_LR_width}) & $C(K, \mathbb{R}^{d_y})$ & ELU,  SELU & $w_{min}\leq \max(2d_x+1, d_y)$ \\
			\cline{2-4}
			\textbf{Ours} (Theorem \ref{thm_LRwidth})& $C([0, 1]^m, \mathbb{R}^{2m})$ & \makecell{ELU, SELU} & $w_{min}=2m+1$\\
			\cline{2-4}
			\textbf{Ours} (Theorem \ref{thm_variant_ReLU_width})& $C(K, \mathbb{R}^{d_y})$ & variants of ReLU\tnote{2} & $d_x+1\leq w_{min}\leq d_x+d_y$\\
			\hline
		\end{tabular}
	}
		\begin{tablenotes}    
			\footnotesize              
			\item[1] They are injective or uniformly approximated by a sequence of injective functions.
			\item[2] Variants of ReLU are LeakyReLU, ELU, CELU, SELU, Softplus. 
		\end{tablenotes} 
	\end{threeparttable}  
\end{table}
\subsection{Related Work}\label{sec_1.3}
We briefly review some studies on the UAP. \cite{UAP_sigmoidal} proved that wide and shallow neural networks with a single hidden layer and sigmoidal activation functions are dense in the space of continuous functions. \cite{UAP_nonpoly} extended this result by demonstrating that the universal approximation property holds for a broader class of non-polynomial activation functions.

Besides such prior studies focused on the approximation ability of wide and shallow networks, several works have investigated that whether a deep narrow neural network can approximate arbitrary functions in $C(\mathbb{R}^{d_x}, \mathbb{R}^{d_y})$, and established upper bounds for the minimum width $w_{min}$ for various activation functions. For the given $d_x$-dimensional input and $d_y$-dimensional output, \cite{Hanin} proved that $w_{min}\leq d_x+d_y$ for ReLU neural networks. \cite{Park} proved that for a new combination of activation functions, networks with width $\max(d_x+1, d_y)$ satisfy the UAP. \cite{Kidger} proved that for continuous and nonpolynomial activations, networks with $d_x+d_y+1$ width can approximate arbitrary continuous functions, while the width result is changed to $d_x+d_y+2$ for nonaffine polynomial activation functions. \cite{Hwang} proved that networks with $\max(2d_x+1, d_y)+\alpha$ width can approximate an arbitrary continuous function in $C(\mathbb{R}^{d_x}, \mathbb{R}^{d_y})$, where $\alpha\in\{0, 1, 2\}$ is a non-negative constant depending on the activation function. 

\subsection{Notations}\label{sec_1.4}
In this subsection, we present the notations used throughout this paper. The notations adopted here are primarily from (\cite{Hwang}; \cite{petersen_DL}).
\begin{itemize}
	\item \textbf{Set of Matrices}: $\mathbb{F}^{m \times n}$ denotes the set of $m \times n$ matrices over a field $\mathbb{F}$, e.g., $\mathbb{F} = \mathbb{R}$ or $\mathbb{C}$. We use $\mathbb{F}^k$ as a shorthand for $\mathbb{F}^{k \times 1}$ and refer to it as the column vector space over $\mathbb{F}$.

	\item \textbf{Set of Full-Rank Matrices}: $\mathbb{F}^{m \times n}_{Full}$ denotes the set of full-rank $m \times n$ matrices over the field $\mathbb{F}$. 
	
	\item \textbf{Column Vector}: $x = (x_1, x_2, \dots, x_n)$ denotes an $n$-dimensional column vector unless otherwise stated, where $x_i$ is the $i$-th component of $x$. Particularly, $\mathbb{R}_+^n$ and $\mathbb{N}_+^n$ denote the sets of $n$-dimensional column vectors whose components are positive real numbers and positive natural numbers, respectively.
	
    \item \textbf{Image}: The image of $A$ under the map $f$, denoted by $f(A)$, is defined as the set of all values attained by $f$ when its argument varies over $A$, i.e., $f(A)=\{f(x)\vert x\in A\}$.
	
	\item \textbf{Compactly Approximate}: $\mathcal{B}\prec \mathcal{A}$ means that $\mathcal{A}\subset C(\mathbb{R}^m, \mathbb{R}^n)$ compactly approximates $\mathcal{B}\subset C(\mathbb{R}^m, \mathbb{R}^n)$ , if for any $f\in \mathcal{B}$, any compact set $K\subset\mathbb{R}^n$, and any $\varepsilon>0$, there exists $ g\in\mathcal{A}$ such that: $\sup_{x\in K}\vert\vert f(x)-g(x)\vert\vert_\infty <\varepsilon$. 
	
	\item \textbf{Equivalence}: $\mathcal{A}\sim\mathcal{B}$ means that $\mathcal{A}\prec \mathcal{B}$ and $\mathcal{B}\prec \mathcal{A}$ are both true.

	\item \textbf{Affine Transformation}: $T_{W, b}$ denotes the affine transformation $T_{W, b} : \mathbb{R}^m \rightarrow \mathbb{R}^n$, defined as
	\begin{align*}
		T_{W, b}(x):=Wx+b, \enspace W\in \mathbb{R}^{n\times m}, b\in \mathbb{R}^{n}.
	\end{align*}
	
	\item \textbf{Set of Affine Transformations}: $\text{Aff}_{m, n}$ denotes the set of affine transformations from $\mathbb{R}^m$ to $\mathbb{R}^n$. $\text{IAff}_{m, n}$ denotes the subset of $\text{Aff}_{m, n}$ consisting of transformations that possess either a left inverse or a right inverse.
	
	\item \textbf{Iterated function}: $f^n$ denotes the iterated function that is obtained by composing $f:\mathbb{R}\rightarrow\mathbb{R}$ with itself $n$ times, i.e., $f^n(x)=f(f^{n-1}(x))=f(\cdots(f(x)))$.
\end{itemize}

\section{Prelimilary}\label{sec_2}
In this section, we introduce the definitions about DNN, and the properties of weight matrices. Throughout this paper, our discussion is based on these definitions.
\subsection{Prelimilary of Deep Neural Networks}\label{sec_2.1}
Before our formal discussion on the minimum width, we first introduce the architecture of a Neural Network (NN) (\cite{petersen_DL}). 
\begin{definition}[Architecture of a Neural Network]\label{defn_NN}
	Let $L\in \mathbb{N}_+$, $d_0$, $\cdots$, $d_{L+1}\in \mathbb{N}_+$, then a function $\Phi_\sigma:\mathbb{R}^{d_0}\rightarrow\mathbb{R}^{d_{L+1}}$ is called a neural network, if there exist matrices $W_i\in \mathbb{R}^{d_{i+1}\times d_i}$, vectors $b_i\in \mathbb{R}^{d_i}$ and activation functions $\sigma_i:\mathbb{R}^{d_{i+1}} \rightarrow \mathbb{R}^{d_{i+1}}$, $i=0, \cdots,L+1$, such that
	\begin{align}
		x^{(0)}&:=x,  \label{eq_x0}\\
		x^{(k+1)}&:=\sigma_k(T_{W_k, b_k}(x^{(k)})) \enspace (\forall k\in\{0, 1,\cdots, L-1 \}), \label{eq_xk}\\
		\Phi_\sigma(x)&:= T_{W_L, b_L}(x^{(L)}) \enspace (\forall x\in \mathbb{R}^{d_0}). \label{eq_NN}
	\end{align}
	When the input is fixed to $x$, we refer to $x^{(0)}$ in Equation (\ref{eq_x0}) as the input layer, and $\Phi_\sigma(x)$ in Equation (\ref{eq_NN}) as the output layer, then $x^{(i)}\enspace(i\in \{1, 2, \cdots, L\})$ are referred to as the hidden layers. The parameters $W_k$ and $b_k$ in Equation ($\ref{eq_xk}$) are called weight matrix and bias vector, respectively.
\end{definition}
Then, according to Definition \ref{defn_NN}, the depth and width of a NN are defined as follows.
\begin{definition}[Depth and Width of a NN]\label{defn_depthandwidth}
	In Definition \ref{defn_NN},  $L$ is referred to as the depth, corresponding to the number of hidden layers. $d_{max} = \max_{i=1, \dots, L} d_i$ is called the width, representing the maximum size of the hidden layers. 
\end{definition}

In Definition \ref{defn_NN}, the activation functions refer to a class of functions $\sigma:\mathbb{R}\rightarrow\mathbb{R}$, which are formally defind as follows.
\begin{definition}[Activation Function]\label{defn_activatefunc}
	Activation function $\sigma:\mathbb{R}\rightarrow\mathbb{R}$ is a piecewise $C^1$-funcion that possesses at least one point $\alpha \in \mathbb{R}$ such that $\sigma'(\alpha) \neq 0$. For notational convenience, it can be applied to vector-valued functions as componentwise operators:
	\begin{align*}
		\sigma(x_1, \cdots, x_n):= (\sigma(x_1), \cdots, \sigma(x_n)).
	\end{align*}
\end{definition}
Several commonly used activation functions are listed in Appendix~\ref{appendix_activationfunc}. Some activation functions, denoted by $\sigma_\beta$, are controlled by an additional parameter $\beta\in\Lambda$, but they still belong to the same class (e.g., LeakyReLU, ELU). We write $\sigma:= \{\sigma_\beta\vert \beta\in \Lambda\}$, which represents the collection of such activation funcions. For example, the set of ELU and LeakyReLU activation functions can be denoted as
\begin{align*}
	ELU:= \{ELU_\beta\vert \beta\in \mathbb{R}_+\}, LeakyReLU:= \{LeakyReLU_\beta\vert \beta\in \mathbb{R}_+\}.
\end{align*}
When the activation function and the size of each layer $x^{(i)} \enspace(i\in \{1, 2, \cdots, L\})$ are given, the set of NNs satisfying these conditions is defined as follows.
\begin{definition}[Set of NNs]\label{defn_set_NN}
	For a given set of activation functions $\sigma=\{\sigma_\beta\vert \beta\in \Lambda\}$, we denote the corresponding set of NNs by $N(\sigma; d_0,  \cdots, d_L, d_{L+1})$. It is defined as follows:
	\begin{align}
		N(\sigma; d_0, \cdots, d_L, d_{L+1}):= &\{\Phi_\sigma:\mathbb{R}^{d_0} \rightarrow \mathbb{R}^{d_{L+1}} \big\vert T_{W_i, b_i} \in \text{Aff}_{d_i, d_{i+1}},\nonumber\\
		&\Phi_\sigma= T_{W_L, b_L} \circ \sigma_L \circ \cdots \circ \sigma_1 \circ T_{W_0, b_0}, \sigma_i\in \sigma\}.
	\end{align}
	If all matrices $W_i(i\in \{0, 1, \cdots, L\})$ are of full rank, we denote the set of NNs by
	\begin{align}
		IN(\sigma; d_0, \cdots, d_L, d_{L+1}):= &\{\Phi_\sigma:\mathbb{R}^{d_0} \rightarrow \mathbb{R}^{d_{L+1}} \big\vert T_{W_i, b_i} \in \text{IAff}_{d_i, d_{i+1}},\nonumber\\
		&\Phi_\sigma= T_{W_L, b_L} \circ \sigma_L \circ \cdots \circ \sigma_1 \circ T_{W_0, b_0}, \sigma_i\in \sigma\}.
	\end{align}
	For notational convenience, we also write
	\begin{align*}
		N(\sigma; d_0,  \cdots, d_L, d_{L+1})&:=N(\{\sigma\}; d_0,  \cdots, d_L, d_{L+1}),\\
		IN(\sigma; d_0,  \cdots, d_L, d_{L+1})&:=IN(\{\sigma\}; d_0,  \cdots, d_L, d_{L+1}),
	\end{align*} 
	when $\sigma$ is a specific activation function (e.g., $ReLU$, $ELU_{0.5}$). 
	
	The set of NN with input dimension $d_x$, output dimension $d_y$ and at most $n$ intermediate dimensions are defined as
	\begin{align*}
		N^\sigma_{d_x, d_y, n}:= \bigcup\limits_{L\in N_+} \bigcup\limits_{1\leq d_1, \cdots, d_L\leq n} N(\sigma; d_x, d_1, \cdots, d_L, d_y).
	\end{align*}
\end{definition}
In fact, the expression of $N^\sigma_{d_x, d_y, n}$ can be simplified by the following proposition.
\begin{proposition}[Equivalence between Sets of NNs]\label{prop_expandNN}
	$N_{m, n, k}^{\sigma} = N(\sigma; m, k, \cdots,k , n)$.
\end{proposition}
\begin{proof}
	The inclusion $N(\sigma; m, k, \cdots,k , n) \subset N_{m, n, k}^{\sigma}$ follows directly from the definitions.
	
	Suppose $f_\sigma:\mathbb{R}^m\rightarrow \mathbb{R}^n$ in $N_{m,n,k}^\sigma$ is represented by
	\begin{align*}
		f_\sigma=T_{W_L, b_L}\circ \sigma\cdots \sigma \circ T_{W_0, b_0}, \enspace W_i\in \mathbb{R}^{d_{i+1}\times d_i}, \enspace b_i\in\mathbb{R}^{d_{i+1}},
	\end{align*}
	where $d_i\leq k$ holds for $i\in\{1, \cdots, L\}$. Define extened weight matrices $\widetilde{W}_i$ and bias vectors $\widetilde{b}_i$ by zero-padding as
	$$\widetilde{W}_i=
	\left(
	\begin{array}{ccc:c}
		& \multirow{3}{*}{$W_i$} & & \\
		& & & 0 \\
		& & &  \\
		\hdashline
		 & 0 &  & 0
	\end{array}
	\right),\enspace
	\widetilde{b}_i=
	\left(
	\begin{array}{c}
		\multirow{3}{*}{$b_i$}\\
		\\
		\\
		\hdashline
		0
	\end{array}
	\right),
	$$
	then the extended NN
	\begin{align*}
		\Phi_\sigma = T_{\widetilde{W}_L, \widetilde{b}_L}\circ \sigma\cdots \sigma \circ T_{\widetilde{W}_0, \widetilde{b}_0} \enspace(\widetilde{W}_i\in \mathbb{R}^{k\times k}, \widetilde{b}_i\in\mathbb{R}^k)
	\end{align*} 
	belongs to $N(\sigma; m, k,\cdots, k, n)$. Hence, it can be concluded that $N_{m, n, k}^\sigma\subset N(\sigma; m, k, \cdots, k, n)$.
\end{proof}
Based on the definitions mentioned above, we now formally define the minimum width, as stated in the following definition. 
\begin{definition}[Minimum Width]\label{defn_minimumwidth}
	The indicator $w_{min}(m, n, \sigma)$ is defined as follows:
	\begin{align*}
		w_{min}(m, n, \sigma):=\min\{ l\in \mathbb{N}\big\vert C(\mathbb{R}^m, \mathbb{R}^n)\prec N_{m, n, l}^{\sigma}\},
	\end{align*}
	it means that any compact set $K\subset \mathbb{R}^m$, a continuous map $f\in C(K, \mathbb{R}^n)$ can be uniformly approximated by $N_{m, n, w_{min}(m, n, \sigma)}^{\sigma}$.
\end{definition}
In Section \ref{sec_3}, we adopt geometric approaches to study the minimum width of networks from the mathematical perspective.
\subsection{Properties of Weight Matrices in DNNs}\label{sec_2.2}
In this subsection, we investigate the properties of weight matrices, showing that any neural network can be approximated by the networks in $IN(\sigma; d_0, \cdots, d_L, d_{L+1})$. By measure theory (\cite{Measure}) and the properties of a real-analytic function, the following theorem shows that the set of full rank matrices is of full Lebesgue measure and dense in $\mathbb{R}^{m\times n}$.
\begin{theorem}[Density of Full-Rank Matrices]\label{thm_fullrank_dense}
	$\mathbb{R}^{m\times n}_{Full}$ is of full Lebesgue measure and dense in $\mathbb{R}^{m\times n}$. 
\end{theorem}
\begin{proof}
	Without loss of generality, we assume that $m \leq n$. A matrix $M \in \mathbb{R}^{m \times n}$ is rank-deficient if and only if every $m \times m$ submatrix $M_1$ of $M$ satisfies $\det(M_1) = 0$. \cite{zeroset_func} proved that for a real-analytic function defined on a connected open domain $U \subset \mathbb{R}^d$, if the function is not identically zero, then its zero set has Lebesgue measure zero. Since the determinant function $\det : \mathbb{R}^{m \times m} \to \mathbb{R}$ is real-analytic, it follows that the set
	\begin{align*}
		\{M_1\vert M_1 \in \mathbb{R}^{m \times m}, \det(M_1) = 0\} \times \mathbb{R}^{(n-m) \times m}
	\end{align*}
	has Lebesgue measure zero in $\mathbb{R}^{m \times n}$. Since each matrix $M \in \mathbb{R}^{m \times n}$ contains at most $\binom{n}{m}$ submatrices $M_1$ of size $m \times m$, it follows that the set of full-rank matrices $\mathbb{R}^{m \times n}_{Full}$ is the complement of the union of at most $\binom{n}{m}$ null sets. Therefore, $\mathbb{R}^{m\times n}_{Full}$ is of full Lebesgue measure and dense in $\mathbb{R}^{m\times n}$.
\end{proof}
By Theorem \ref{thm_fullrank_dense}, the set $\mathbb{R}^{d_{i+1}\times d_i}_{Full}$ is dense in $\mathbb{R}^{d_{i+1}\times d_i}$. Therefore, for each $i\in\{0, 1, \cdots, L\}$, we can find $W_i^\Phi\in\mathbb{R}^{d_{i+1}\times d_i}_{\text{Full}}$ arbitrarily close to $W_i^\Psi$, such that any NN $\Psi_\sigma \in N(\sigma; d_0, \cdots, d_L, d_{L+1})$ can be approximated by $IN(\sigma; d_0, \cdots, d_L, d_{L+1})$.
\begin{theorem}[UAP of Networks with Full-Rank Weight Matrices]\label{thm_INNapprox}
	For a given compact domain $K$, any NN can be approximated by the set of NNs with full rank weight matrices. In other words,
	\begin{align*}
		N(\sigma; d_0, \cdots, d_L, d_{L+1}) \prec IN(\sigma; d_0, \cdots, d_L, d_{L+1}).
	\end{align*}
\end{theorem}
\begin{proof}
	Let $K\subset \mathbb{R}^{d_0}$ be a fixed compact domain, any NN $\Psi_\sigma\in N(\sigma; d_0, \cdots, d_L, d_{L+1})$ can be presented as:
	\begin{align*}
		\Psi_\sigma=T_{W_L^\Psi, b_L^\Psi} \circ \sigma\circ \cdots \circ \sigma\circ T_{W_0^\Psi, b_0^\Psi}.
	\end{align*}
	Since $K$ is compact, for any $\varepsilon>0$, there exists $\delta>0$ such that if all the matrices $W_i^\Phi\subset\mathbb{R}_{Full}^{d_{i+1}\times d_i}$ satisfy that $\vert\vert W_i^\Phi-W_i^\Psi\vert\vert_\infty <\delta$, $i=0, 1, \cdots, L$, and the bias vectors are chosen as $b_i^\Phi=b_i^\Psi$, $i=0, 1, \cdots, L$, then
	\begin{align*}
		\Phi_\sigma=T_{W_L^\Phi, b_L^\Phi} \circ \sigma\circ \cdots \circ \sigma\circ T_{W_0^\Phi, b_0^\Phi},\\
		\vert \vert \Phi_\sigma(x)-\Psi_\sigma(x)\vert\vert_\infty<\varepsilon\enspace(\forall x\in K).
	\end{align*}
	By Theorem \ref{thm_fullrank_dense}, $\mathbb{R}^{d_{i+1}\times d_i}_{Full}$ is dense in $\mathbb{R}^{d_{i+1}\times d_i}$, so we can always find a $W_i^\Phi\subset\mathbb{R}^{d_{i+1}\times d_i}_{Full}$ that close enough to the $W_i^\Psi$. So any $\Psi_\sigma\in N(\sigma; d_0, \cdots, d_L, d_{L+1})$ can be approximated by $IN(\sigma; d_0, \cdots, d_L, d_{L+1})$. 
\end{proof}
Consequently, it is sufficient to consider the networks in $IN(\sigma; m, d_1, \cdots, d_L, n)$ to investigate the lower bounds of $w_{min}(m, n, \sigma)$ for different activation functions. Since all weight matrices of a NN $\Phi\in IN(\sigma; m, d_1, \cdots, d_L, n)$ are non-degenerate, it becomes easier for us to analyze the geometric properties of the outputs as discussed in Subsection \ref{sec_3.1}.

\section{Minimum Width for Universal Approximation}\label{sec_3}
Throughout this section, we investigate the minimum width $w_{min}(m, n, \sigma)$ for some activation functions $\sigma$, when the input dimension and the output dimension are $m$ and $n$, respectively. 

In Subsection \ref{sec_3.1}, we construct a more intuitive example via a new geometric approach based on the Poincar\'e-Miranda Theorem (\cite{Poincare-Miranda}), and show the lower bound of minimum width that 
\begin{align*}
	w_{min}(m, n, \sigma) \geq n+\mathbf{1}_{m<n\leq 2m}
\end{align*}
for networks with injective activation functions and those uniformly approximated by a sequence of injective functions (Theorem \ref{thm_injectivewidth}), which is consistent with the statement in (\cite{Kim}). 

In Subsection \ref{sec_3.2}, we show our main results of minimum width, for networks with variants of ReLU (LeakyReLU, ELU, CELU, SELU, Softplus). We prove the inequality of minimum width
\begin{align*}
	w_{min}(m, n, \sigma)\leq \max(2m+1, n)
\end{align*}
for networks with ELU and SELU (Theorem \ref{thm_variant_LR_width}), which extends the range of activation functions in (\cite{Hwang}). Furthermore, we prove the equality of minimum width
\begin{align*}
	w_{min}(m, 2m, \sigma)= 2m+1
\end{align*}
for networks with ELU and SELU (Theorem \ref{thm_LRwidth}). In addition, we show the result
\begin{align*}
	m+1\leq w_{min}(m, n, \sigma)\leq m+n
\end{align*}
for networks with variants of ReLU (Theorem \ref{thm_variant_ReLU_width}), which extends the range of activation functions in (\cite{Hanin}).

\subsection{A New Proof of the Minimum Width for Injective Activation Functions}\label{sec_3.1}
\cite{Kim} proved the existence of the example that can not be approximated by networks with injective activation functions and width less than $d_y+\mathbf{1}_{d_x<d_y\leq 2d_x}$, by Brouwer Fixed Point Theorem, Tietze Extension Lemma and the Pasting Lemma (\cite{Munkres}). In this subsection, we construct a new and intuitive example, as stated in Equation (\ref{eq_gstar}). Then we give a new proof of the lower bound of minimum width, for networks with injective activation functions and those uniformly approximated by a sequence of injective functions, by utilizing this intuitive example.

From the definition of the neural networks (Definition \ref{defn_NN}), it follows that a network is essentially a composition of linear and nonlinear transformations. The input dataset can be represented well by the set of points in Euclidean space. In view of elementary geometry theory, linear layers translate, rotate, rescale and reflect the datasets, while the activation layers introduce geometric distortion to the datasets, resulting in curved structure. Since the geometric properties of NNs are closely related to topological embeddings (\cite{Munkres}), we begin by introducing the definition of the topological embedding.
\begin{definition}[Topological Embedding]
	If $f:X\rightarrow Y$ is an injective continuous map, and $f^{\prime}:X\rightarrow f(X)$ obtained by restricting the range of $f$ happens to be a homeomorphism of $X$ with $f(X)$, then $f:X\rightarrow Y$ is called an embedding.
\end{definition}
By the definition of NNs (Definition \ref{defn_NN}), for an injective activation function $\sigma$ and an input $K\subset \mathbb{R}^{d_x}$, it is evident that if $m\leq d_1\leq\cdots\leq d_L$, any network $\Phi_\sigma\in IN(\sigma; m, d_1, \cdots, d_L, n)$ is a topological embedding. It motivates us to construct examples with self-intersections, which contradict the properties of topological embeddings. We begin our discussion by introducing the following Poincar\'e-Miranda Theorem.
\begin{theorem}[Poincar\'e-Miranda Theorem]\label{thm_PM}
	Let $I^n=[a_1, b_1]\times \cdots \times [a_n, b_n]\subset \mathbb{R}^n$, consider a continuous mapping
	\begin{align*}
		f:I^n\rightarrow \mathbb{R}^n, \enspace f(x)=(f_1(x), \cdots, f_n(x)), \enspace x=(x_1, \cdots, x_n),
	\end{align*}
	if it satisfies
	\begin{align}
		h_if_i(x_1, \cdots, a_i, \cdots,  x_n)\geq 0, \enspace h_if_i(x_1, \cdots, b_i, \cdots, x_n)\leq 0\enspace(\forall x\in I^n)\label{eq_PMthm}
	\end{align}
	for $i=1, 2, \cdots, n$, where $h_i\in \{-1, 1\}$, then there exists $\bar{x}\in I^n$ such that $f(\bar{x})=0$.
\end{theorem}
Moreover, it is evident that the condition in Equation (\ref{eq_PMthm}) can be modified to
\begin{align}
	&[f_i(x_1,\cdots, a_i, \cdots, x_n)- f_i(x_1, \cdots, x_i, \cdots, x_n)]*  \nonumber\\
	&[f_i(x_1, \cdots, b_i, \cdots, x_n)- f_i(x_1, \cdots, x_i, \cdots, x_n)]\leq 0 \enspace (\forall x\in I^n). \label{eq_PMthm_product}
\end{align}

By Poincar\'e-Miranda Theorem, we can determine whether the output of a continuous map has self-intersections, as stated in the following lemma.
\begin{lemma}[Sufficient Conditions for Self-intersections]\label{lem_appro_indimn}
	Consider a continuous map $g:[0, 1]^m\rightarrow \mathbb{R}^{2m}$, defined as 
	\begin{align*}
		g(t)=(g_1(t), g_2(t), \cdots,g_{2m-1}(t), g_{2m}(t)),
	\end{align*}
	where $t=(t_1, t_2, \cdots, t_m)\in [0, 1]^m$. If there exist a point $P\in g([0, 1]^m)$ and an open neighborhood $U$ of $P$ such that
	\begin{enumerate}[(1)]
		\item $g^{-1}(g([0, 1]^m)\cap \overline{U})=\prod\limits_{i=1}^{m}([a_{i1}, a_{i2}] \cup [b_{i1}, b_{i2}])$ with $[a_{i1}, a_{i2}]\cap [b_{i1}, b_{i2}]=\emptyset$ for each $i\in \{1, 2,\cdots, m\}$,
		\item $[g_{2k-1}(x_1, \cdots, a_{k1}, \cdots, x_m) - g_{2k-1}(y_1, \cdots, y_k, \cdots, y_n)]*[g_{2k-1}(x_1, \cdots, a_{k2}, \cdots, x_m)-g_{2k-1}(y_1, \cdots, y_k, \cdots, y_m)]<0$ holds for all $x_i\in[a_{i1}, a_{i2}], y_j\in [b_{j1}, b_{j2}]$,
\item $[g_{2k}(x_1, \cdots, x_k, \cdots, x_m) - g_{2k}(y_1, \cdots, b_{k1}, \cdots, y_m)]*[g_{2k}(x_1, \cdots,x_k, \cdots, x_m)-g_{2k}(y_1, \cdots, \linebreak b_{k2}, \cdots, y_m)]<0$ holds for all $x_i\in[a_{i1}, a_{i2}], y_j\in [b_{j1}, b_{j2}]$.
	\end{enumerate}
	Then there exists $\varepsilon>0$, such that for any continous map $f:[0, 1]^m\rightarrow \mathbb{R}^{2m}$, if it satisfies that 
	\begin{align*}
		\sup_{t\in [0, 1]^m}\vert\vert f(t)-g(t)\vert \vert_{\infty}<\varepsilon,
	\end{align*}
	then there exist $\bar{t}_1\neq \bar{t}_2$ such that $f(\bar{t}_1)=f(\bar{t}_2)$, i.e., $f:[0, 1]^m\rightarrow\mathbb{R}^{2m}$ can not be injective.
\end{lemma}
\begin{proof}
	The proof is provided in Appendix \ref{appendix_appro_indimn}.
\end{proof}
When the activation function $\sigma$ is injective, any $\Phi_\sigma\in IN(\sigma; m, k, \cdots, k, n)$ is a topological embedding whenever $m\leq k\leq n$. By Lemma \ref{lem_appro_indimn}, we show the sufficient condition for the existence of self-intersections.
\begin{theorem}[Sufficient Condition for Inapproximation]\label{thm_m<n<=2m}
	For an injective activation function $\sigma$, when $m<n\leq 2m$, we have $C(\mathbb{R}^m, \mathbb{R}^{n})\nprec N_{m, n, n}^{\sigma}$.
\end{theorem}
\begin{proof}
	By Theorem \ref{thm_INNapprox}, it is sufficient to discuss whether a $\Phi_\sigma \in IN(\sigma; m, n, \cdots, n)$ can approximate the target continuous function in $C(\mathbb{R}^m, \mathbb{R}^{n})$.
	
	When $m< n\leq2m$, we denote any $\Phi_\sigma\in IN(\sigma; m, n, \cdots, n)$ by
	 \begin{align*}
	 	\Phi_\sigma(x)=(\Phi_1(x), \cdots, \Phi_n(x)), x=(x_1, x_2,\cdots, x_m)\in \mathbb{R}^{m},
	 \end{align*}
	where $\sigma$ is an injective activation function. Then we can construct 
	\begin{align}\label{eq_g}
		g(x)= (g_1(x), g_2(x), \cdots, g_{2m}(x)),
	\end{align}
	where the $g_i:[0, 1]^m \rightarrow \mathbb{R}$ are defined as
	\begin{align*}
			\scriptsize{
				g_{2k-1}(t_1, \cdots, t_m)=
			\left\{
			\begin{aligned}
				10t_k-1 \quad (0\leq t_k\leq \frac{3}{10})\\
				2 \quad (\frac{3}{10} \leq t_k \leq \frac{1}{2})\\
				7-10t_k \quad (\frac{1}{2}\leq t_k \leq \frac{7}{10})\\
				0 \quad (\frac{7}{10}\leq t_k \leq 1)
			\end{aligned}
			\right.
			}&,\\
			\scriptsize{
				g_{2k}(t_1, \cdots, t_m)=
			\left\{
			\begin{aligned}
				0 \quad (0\leq t_k\leq \frac{3}{10})\\
				10t_k-3 \quad (\frac{3}{10} \leq t_k \leq \frac{1}{2})\\
				2 \quad (\frac{1}{2}\leq t_k \leq \frac{7}{10})\\
				9-10t_k \quad (\frac{7}{10}\leq t\leq 1)
			\end{aligned}.
			\right.
		}&
	\end{align*}
	For $h:\mathbb{R}^m \rightarrow \mathbb{R}^{2m}$ defined by
	\begin{align*}
		h(x)=(\Phi_1(x), \cdots, \Phi_n(x), g_{n+1}(x), \cdots, g_{2m}(x)),
	\end{align*}
	since any $\Phi_\sigma\in IN(\sigma; m, n, \cdots, n)$ is a topological embedding, it is evident that $h$ is also a topological embedding.
	
	Let the neighborhood $U$ in Lemma \ref{lem_appro_indimn} be $\{x\in \mathbb{R}^{2m}\big\vert \vert\vert x-(0, \cdots, 0)\vert\vert_\infty < 1\}$,
	then we know that the corresponding parameters are $[a_{i1}, a_{i2}]=[0, \frac{1}{5}], [b_{j1}, b_{j2}]=[\frac{4}{5}, 1]$, $i=1, 2, \cdots, m, j=1, 2, \cdots, m$. By checking the product condtions, it is evident that when $\varepsilon$ is extremely small, any injective continuous map $h:\mathbb{R}^m \rightarrow \mathbb{R}^{2m}$ can not approximate the $g=(g_1, g_2, \cdots, g_{2m-1}, g_{2m})$ with the error under $\varepsilon$, which implies that $IN(\sigma; m, n, \cdots, n)$ can not approximate the 
\begin{align}\label{eq_gstar}
	g^*=(g_1, g_2, \cdots, g_{n}).
\end{align}
\end{proof}
The self-intersections can be visualized in relatively low-dimensional cases. Figure \ref{fig_self-intersection} depicts that for any network with an injective activation function, the image $\Phi_\sigma([0, 1])$ has a self-intersection point when $\Phi_\sigma$ attempts to approximate $g=(g_1, g_2)$. 
\begin{figure}[htbp]
	\centering
	\includegraphics[width=0.6\linewidth, height=0.6\linewidth]{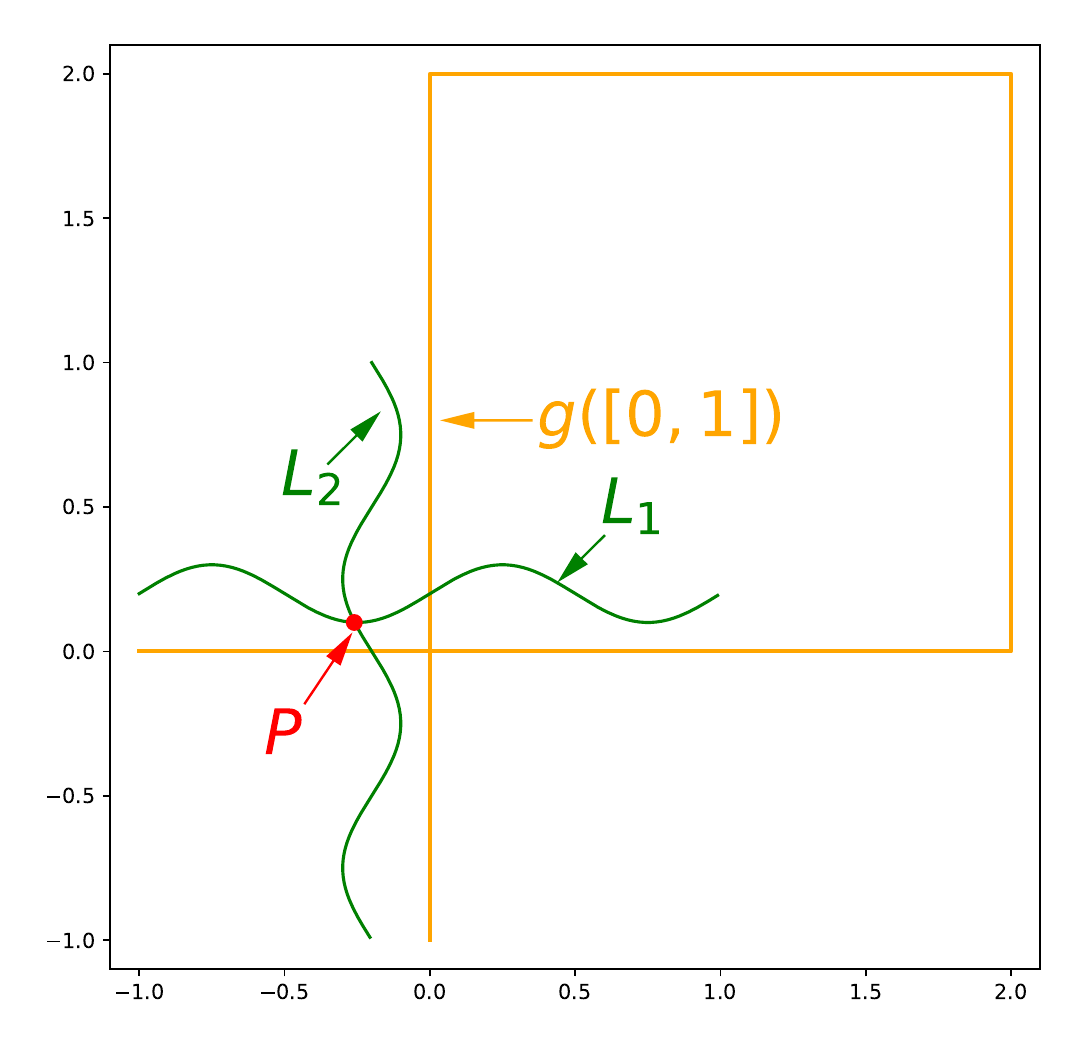}
	\caption{Self-intersection structure of $\Phi_\sigma$. The \textcolor{orange}{orange} polyline represents $g([0, 1])$, the \textcolor{green}{green} curves $L_1$ and $L_2$ denote $\Phi_\sigma([0, \frac{1}{5}])$ and $\Phi_\sigma([\frac{4}{5}, 1])$, respectively. Then $L_1$ must intersects with $L_2$ at the \textcolor{red}{red} point $P$, which contradicts the property that the network $\Phi_\sigma$ is a topological embedding.}
	\label{fig_self-intersection}
\end{figure}

By Theorem \ref{thm_m<n<=2m}, we determine the lower bound of the minimum width of networks with injective activation functions (e.g., ELU, LeakyReLU, Softplus) and those can be uniformly approximated by a sequence of injective functions (e.g., ReLU, HardTanh (\cite{HardTanh}), ReLU6 (\cite{ReLU6})), as stated in the following theorem.
\begin{theorem}[Minimum Width for Injective Activation Functions]\label{thm_injectivewidth}
	For activation function $\sigma$ that is injective or uniformly approximated by a sequence of injective functions, we have
	\begin{align*}
		w_{min}(m, n, \sigma) \geq n+\mathbf{1}_{m<n\leq 2m}.
	\end{align*}
\end{theorem}
\begin{proof}
	We suppose that $\sigma$ is injective, and discuss it in different dimension cases.
	
	Case 1: $n\leq m$ or $n\geq 2m+1$.
	
	In this case, \cite{Cai} proved that for an arbitrary activation function $\sigma$, $w_{min}(m, n, \sigma)\geq \max(m, n)$, so it is obvious that $w_{min}(m, n, \sigma) \geq n$.
	
	Case 2: $m< n\leq 2m$.
	
	By Theorem \ref{thm_m<n<=2m}, $N_{m, n, n}^\sigma$ can not approximate any $f: \mathbb{R}^m\rightarrow \mathbb{R}^n$, such that $w_{min}(m, n,\sigma)\geq n+1$.
	
	When $\sigma$ can be uniformly approximated by a sequence of injective functions $\{f_i\}_{i\in \mathbb{N}^+}$, it is evident that for any compact set $K$ and $\varepsilon>0$, there exists an index $i\in \mathbb{N}^+$ such that $\sup\limits_{x\in K}\vert\vert \sigma(x)- f_i(x) \vert\vert<\varepsilon$, which implies that the conclusion remains true.
\end{proof}

\subsection{Minimum Width of  Networks with Variants of ReLU}\label{sec_3.2}
Due to the formula of ReLU, it suffers from the vanishing gradient problem for negative input values. Since variants of ReLU  (LeakyReLU, ELU, CELU, SELU and Softplus) introduce smooth curves for both positive and negative inputs and improve the gradient, they have attracted significant research interest. In this subsection, we show our main minimum width results of variants of ReLU, by analyzing the images of one-dimension input under variants of ReLU.

Geonho Hwang (\cite{Hwang}) researched the upper bound of $w_{min}(m, n, LeakyReLU)$. In fact, we can utilize the geometric property that LeakyReLU can bend a strainght line into a polyline, by considering the one-dimensional input. Thus another LeakyReLU with different parameter can be approximated after finitely many bendings, under an arbitrary approximation accuracy tolerance $\varepsilon>0$. Therefore, when all layers employ the same activation function $LeakyReLU_\beta$, with $\beta\in (0, 1)\cup (1+\infty)$ fixed, the minimum width remains unchanged.
\begin{theorem}[Equivalence between LeakyReLU Networks]\label{thm_Leaky_equivalent}
	For any fixed $\beta\in (0, 1)\cup(1, +\infty)$, we have $N_{m, n, k}^{LeakyReLU_\beta}\sim N_{m, n, k}^{LeakyReLU}$. It implies that 
	\begin{align*}
		w_{min}(m, n, LeakyReLU_\beta) = w_{min}(m, n, LeakyReLU).
	\end{align*}
\end{theorem}
\begin{proof}
	The proof is provided in Appendix \ref{appendix_Leaky_equivalent}.
\end{proof}
For an approximation accuracy tolerance $\varepsilon=0.3$, Figure \ref{fig_LR_to_LR} illustrates the process that how a sequence of LeakyReLU NNs $\{\Phi_i\vert i\in\mathbb{N}_+\}$ in $N_{1, 1, 1}^{LeakyReLU_{0.2}}$ satisfy the approximation condition that $\vert \Phi_i-LeakyReLU_{0.1}\vert <\varepsilon$ holds for any $i\in \mathbb{N}_+$ and a set $K_i=[-3i, +\infty)$. We know that $K_i\subset K_{i+1}$ holds for all $i\in \mathbb{N}_+$, and for any compact set $K\subset\mathbb{R}$, there exists $K_i$ such that $K\subset K_i$. By Appendix \ref{appendix_Leaky_equivalent}, for any $\varepsilon>0$ and fixed $\alpha, \beta\in(0, 1)$, we have $LeakyReLU_\alpha\prec N_{1, 1, 1}^{LeakyReLU_\beta}$, which implies that $N_{m, n, k}^{LeakyReLU}\prec N_{m, n, k}^{LeakyReLU_\beta}$.
 \begin{figure}[htbp]
 	\centering
 	\subfigure[Diagram of $\Phi_2$]{
 		\includegraphics[width=0.7\linewidth, height=0.28\linewidth]{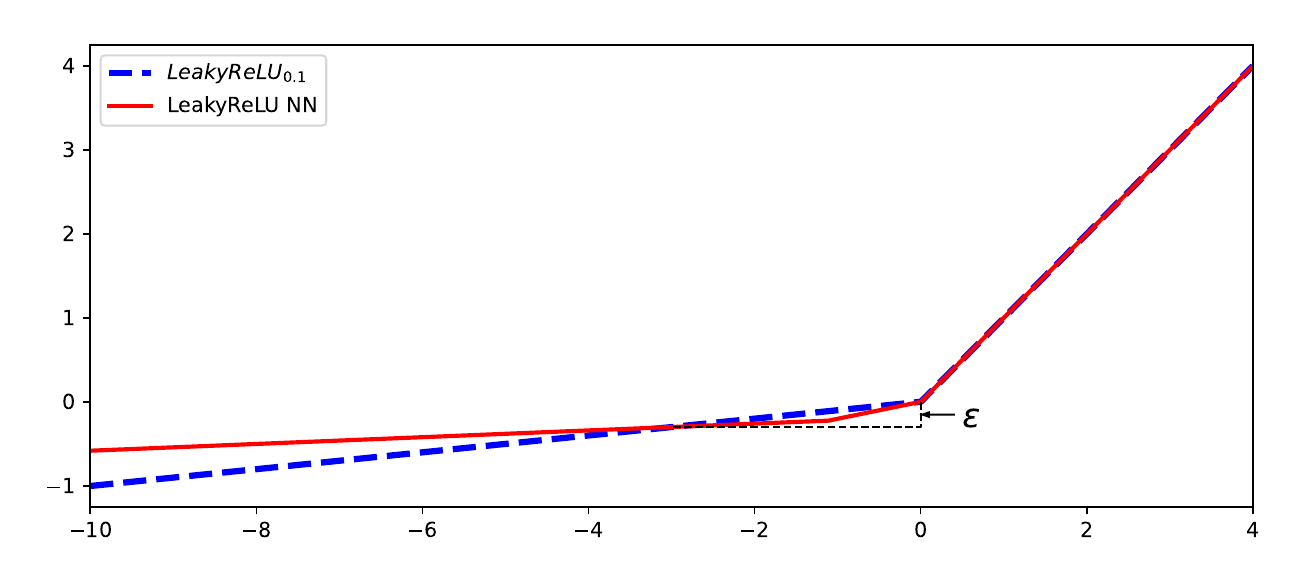}
 	}
 	\subfigure[Diagram of $\Phi_4$]{
 		\includegraphics[width=0.7\linewidth, height=0.28\linewidth]{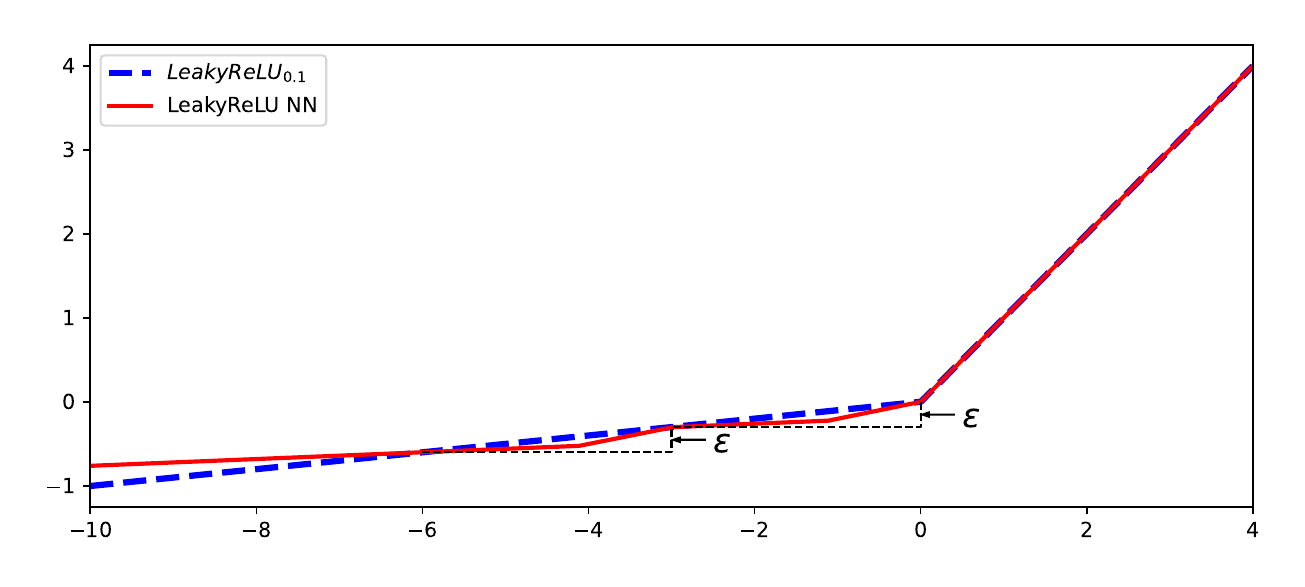}
 	}
 	\subfigure[Diagram of $\Phi_6$]{
 		\includegraphics[width=0.7\linewidth, height=0.28\linewidth]{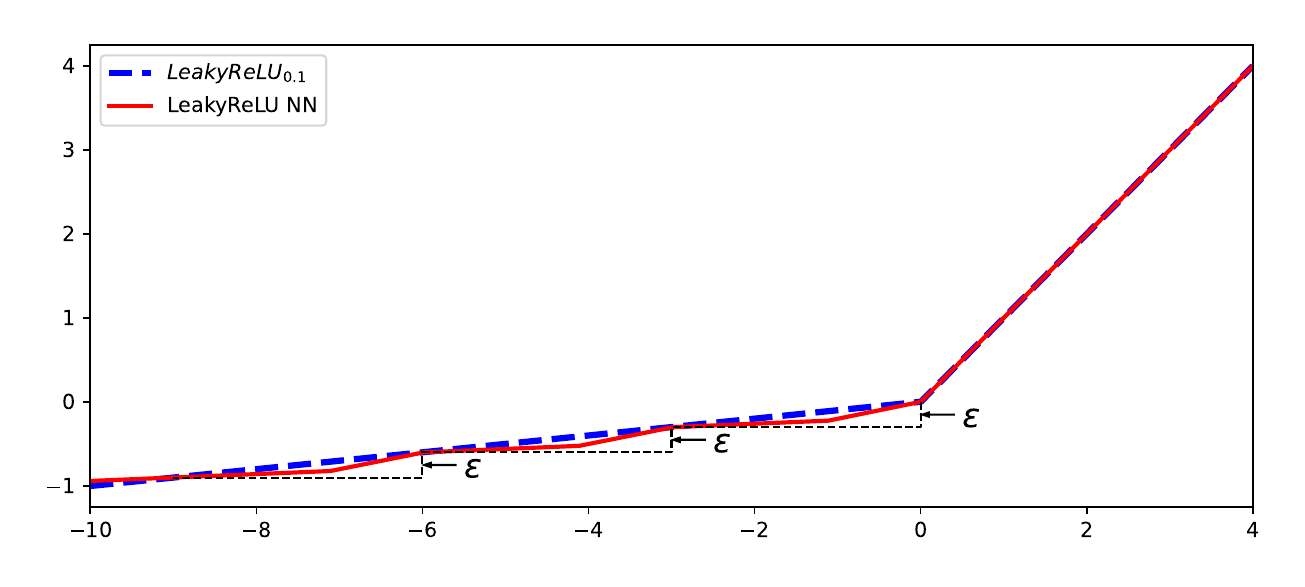}
 	}
 	\caption{The diagram of the geometric transformations of LeakyReLU layers. Subfigure (a), Subfigure (b) and Subfigure (c) adopt the construction of $\Phi_2, \Phi_4, \Phi_6$, respectively, which are mentioned in Appendix \ref{appendix_Leaky_equivalent}. The parameters are $\alpha=0.1$, $\beta=0.2$, $\beta_1=\beta^2=0.04$, $\beta_2=\beta^1=0.2$. They achieve the desired accuracy over $K_1=[-3, +\infty), K_2=[-6, +\infty)$ and $K_3=[-9, +\infty)$, respectively. }
 	\label{fig_LR_to_LR}
 \end{figure}
 
However, though they requires the same network width, the depth may differ significantly. By the definitions of $N_{m, n, k}^{LeakyReLU}$ and $N_{m, n, k}^{LeakyReLU_\beta}$, it is evident that the depth required by $N_{m, n, k}^{LeakyReLU}$ does not exceed that required by $N_{m, n, k}^{LeakyReLU_\beta}$. By a simple example, the following lemma shows that $N_{m, n, k}^{LeakyReLU_\beta}$ requires greater network depth than that required by $N_{m, n, k}^{LeakyReLU}$ in certain cases.

\begin{lemma}[Disadvantage of Fixed Parameters]\label{lem_Leaky_depth}
	There exists $F:[-1, 1]\rightarrow\mathbb{R}$ that $F\prec N(LeakyReLU; 1, 1, 1)$, and $F\nprec N(LeakyReLU_{0.1}; 1, 1, 1)$, which implies that a LeakyReLU network with fixed parameters across all layers requires greater network depth in certain cases.
\end{lemma}
\begin{proof}
	Let $F:[-1, 1]\rightarrow \mathbb{R}$ be defined as:
	$$F(x)=\left\{
	\begin{aligned}
		x, & \quad x\in [0, 1],\\
		0.2x, & \quad x\in [-1, 0].
	\end{aligned}
	\right.
	$$
	It is easy to verify that any $\Phi(x)\in N(LeakyReLU_{0.1}; 1, 1, 1)$ takes one of the following two forms:
	\begin{enumerate}[(1)]
		\item $$
		\Phi(x)=
		\left\{
		\begin{aligned}
			&a x + b ,& -1\leq x\leq c,\\
			&10a x +  b - 9 ac,&  c\leq x\leq 1,\\
		\end{aligned}
		\right.
		$$
		\item $$
		\Phi(x)=
		\left\{
		\begin{aligned}
			&10a x + b ,& -1\leq x\leq c,\\
			&a x +  b + 9 ac,&  c\leq x\leq 1,\\
		\end{aligned}
		\right.
		$$
	\end{enumerate}
	where $a, b, c\in \mathbb{R}$ ($c=-1$ and $c=1$ are allowed). Then $\Phi(x)-F(x)$ is a continuous piecewise linear function defined on at most 3 disjoint intervals. Without loss of generality, we assume that $\Phi(x)$ has the form in the case (1), then the proof of the case (2) is analogous. We suppose that $-1\leq c \leq 0 \leq 1$. Since $\Phi(x)-F(x)$ is piecewise linear and continuous, it follows that
	\begin{align*}
		\vert \Phi(x) - F(x)\vert < \varepsilon, \enspace x\in [-1, 1] \Leftrightarrow \vert \Phi(x)-F(x) \vert< \varepsilon, \enspace x=-1, c, 0, 1.
	\end{align*}
	Consequently, for any $\varepsilon>0$, $F\prec \Phi$ means there exist $a, b, c\in \mathbb{R}$ such that
	\begin{enumerate}[(1)]
		\item $-\varepsilon <-a+b+0.2 < \varepsilon$,
		\item $ -\varepsilon <ac+b-0.2c < \varepsilon$,
		\item $ -\varepsilon < b-9ac < \varepsilon$,
		\item $-\varepsilon < 10a+b-9ac- 1 <\varepsilon$.
	\end{enumerate}
	Since $-22\varepsilon<10*[(-a+b+0.2) - (ac+b-0.2c)] + [(b-9ac)-(ac+b-0.2c)]<22\varepsilon$, we have $-22\varepsilon\leq -1.8c-2\leq 22 \varepsilon$. When $\varepsilon<\frac{1}{220}$, it contradicts the assumption $-1\leq c\leq 0$.
	
	The proof of that $-1\leq 0\leq c\leq 1$ is analogous. Hence, it suffices to verify that for any $c\in[-1, 1]$, it is impossible that $\vert \Phi(x)-F(x)\vert <\varepsilon$ holds for $x=-1, 0, 1, c$. Therefore, $F\nprec\Phi$, which implies that a LeakyReLU network with fixed parameters across all layers requires greater network depth in this case.
\end{proof}
When the network width is fixed, restricting the depth results in fewer parameters, lower computational complexity, and faster training. Therefore, in the following discussion, we focus on the LeakyReLU NNs in $N(LeakyReLU; m, d_1, \cdots, d_L, n)$ rather than those in $N(LeakyReLU_\beta; m, d_1, \cdots, d_L, n)$.  

Motivated by the proof of Theorem \ref{thm_Leaky_equivalent}, we can utilize the geometric property that ELU and SELU activation function can bend a curve, when considering the one-dimensional input. Therefore, we can adjust the parameters of activation function, such that the networks with width 1 can approximate another activation function. 
\begin{theorem}[Equivalence between Variants of ReLU]\label{thm_variant_ReLU_sim}
	$N_{m, n, k}^{ELU}\sim N_{m, n, k}^{LeakyReLU}\sim N_{m, n, k}^{SELU}$
\end{theorem}
\begin{proof}
	The proof is provided in Appendix \ref{appendix_variant_ReLU_sim}.
\end{proof}
For an approximation accuracy tolerance $\varepsilon=0.3$, Figure \ref{fig_ELU_to_LR} illustrates the process that how a sequence of ELU NNs $\{\Phi_i\vert i\in\mathbb{N}_+\}$ in $N_{1, 1, 1}^{ELU}$ satisfys the approximation condition, that $\vert \Phi_i-LeakyReLU_{0.1}\vert <\varepsilon$ holds for any $i\in \mathbb{N}_+$ and a set $K_i=[-3i,+\infty)$. We know that $K_i\subset K_{i+1}$ holds for any $i\in \mathbb{N}_+$, and for any compact set $K\subset\mathbb{R}$, there exists a $K_i$ such that $K\subset K_i$. By Appendix \ref{appendix_variant_ReLU_sim}, for any $\varepsilon>0$ and fixed $\alpha\in(0, +\infty)$, we have $LeakyReLU_\alpha\prec N_{1, 1, 1}^{ELU}$, which implies that $N_{m, n, k}^{LeakyReLU}\prec N_{m, n, k}^{ELU}$.
\begin{figure}[htbp]
	\centering
	\subfigure[Diagram of $\Phi_1$]{
		\includegraphics[width=0.7\linewidth, height=0.28\linewidth]{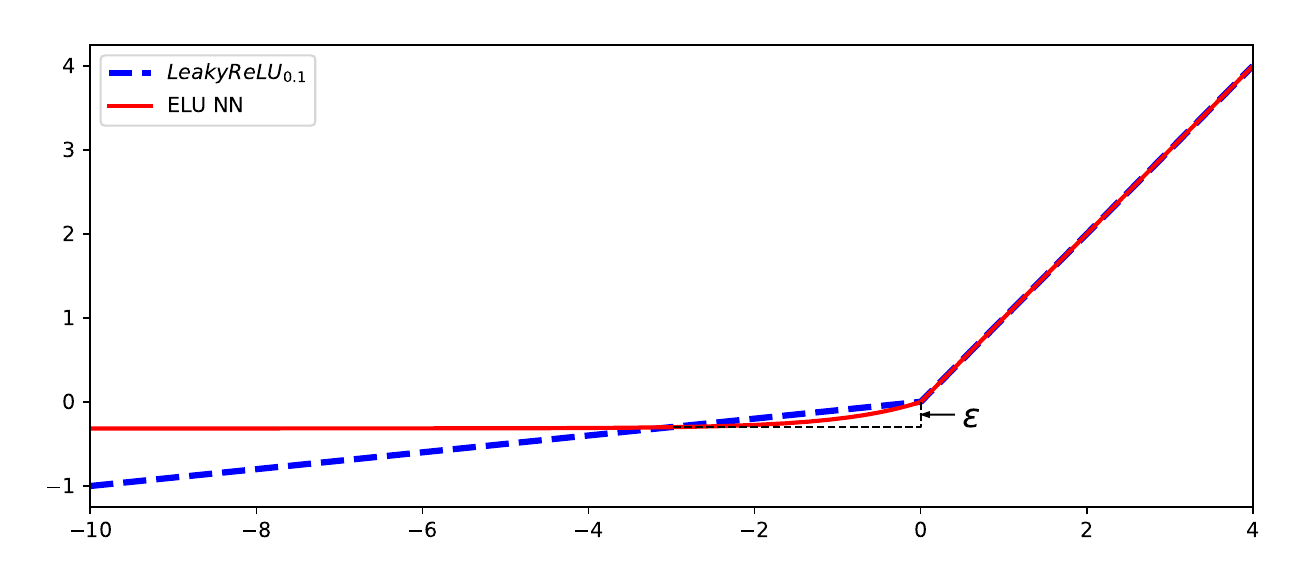}
	}
	\subfigure[Diagram of $\Phi_2$]{
		\includegraphics[width=0.7\linewidth, height=0.28\linewidth]{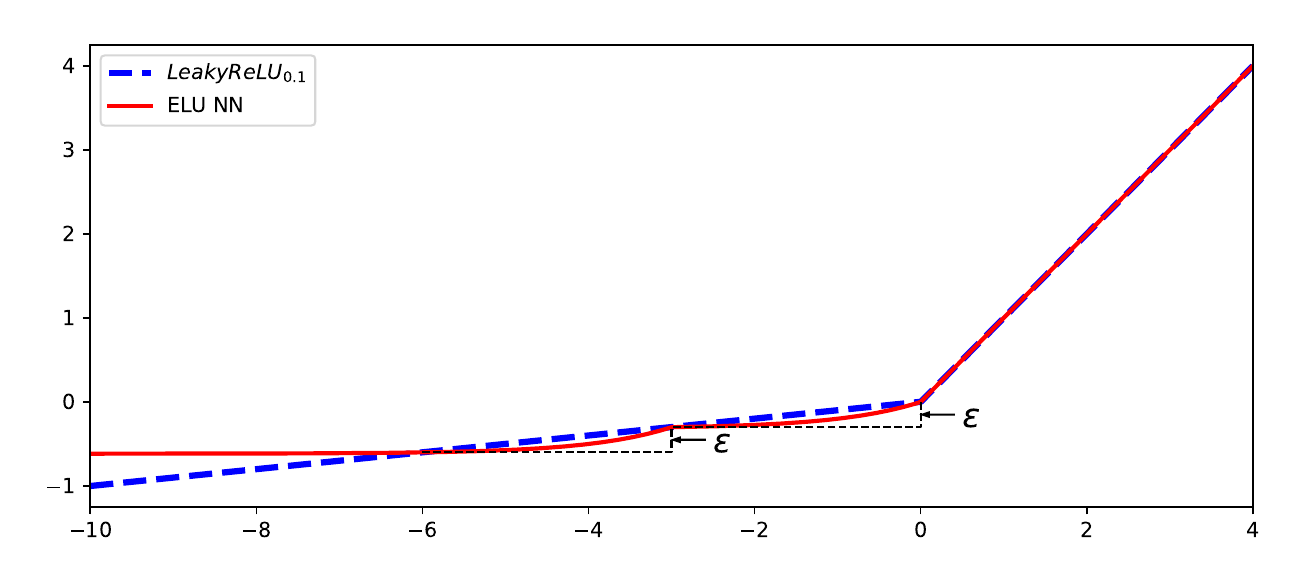}
	}
	\subfigure[Diagram of $\Phi_3$]{
		\includegraphics[width=0.7\linewidth, height=0.28\linewidth]{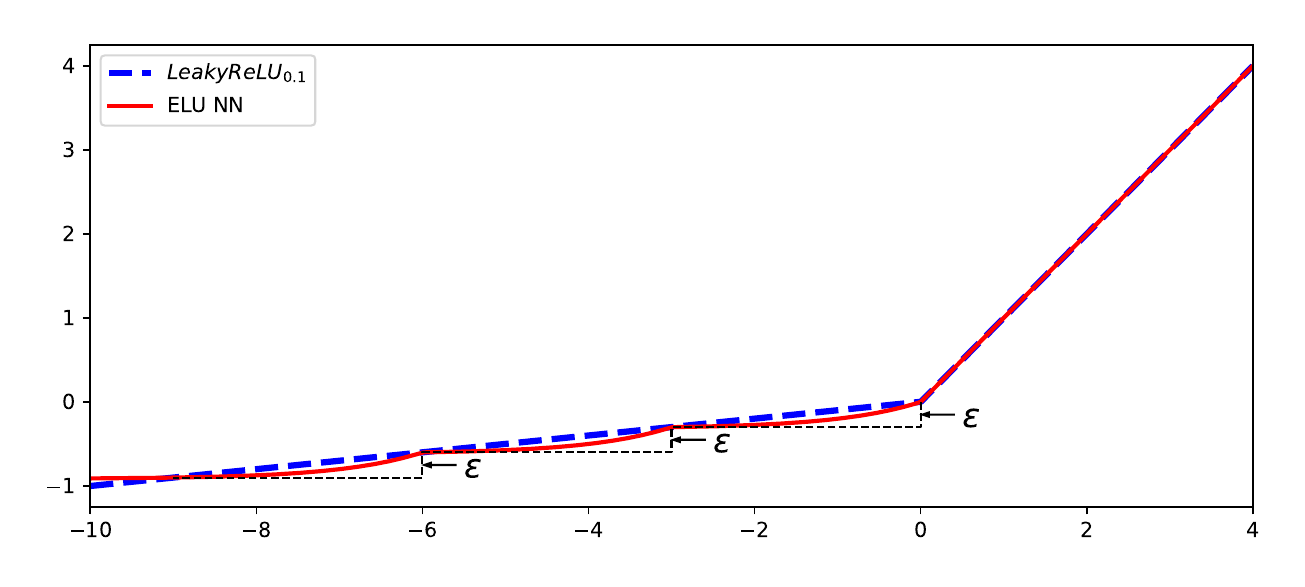}
	}
	\caption{The diagram of the geometric transformations of ELU layers. Subfigure (a), Subfigure (b) and Subfigure (c) adopt the construction of $\Phi_1, \Phi_2, \Phi_3$, respectively, which are mentioned in Appendix \ref{appendix_variant_ReLU_sim}. They achieve the desired accuracy over $K_1=[-3, +\infty), K_2=[-6, +\infty)$ and $K_3=[-9, +\infty)$, respectively. }
	\label{fig_ELU_to_LR}
\end{figure}

By Theorem \ref{thm_injectivewidth}, we prove the upper bound of minimum width of network with ELU and SELU, which refines the results in (\cite{Hwang}).
\begin{theorem}[Upper Width Bounds for Variants of ReLU]\label{thm_variant_LR_width}
	For the activation function $LeakyReLU_\beta \left(\beta\in(0, 1)\cup (1, +\infty)\right)$, $ELU$, $SELU$, we have
	\begin{align*}
		w_{min}(m, n, \sigma)\leq \max(2m+1, n).
	\end{align*}
\end{theorem}
\begin{proof}
	By Theorem 20 in (\cite{Hwang}), we have $w_{min}(m, n, LeakyReLU)\leq \max(2m+1, n)$. By Theorem \ref{thm_Leaky_equivalent}, we have $w_{min}(m, n, LeakyReLU_\beta)=w_{min}(m, n, LeakyReLU)\leq \max(2m+1, n)$. By Theorem \ref{thm_variant_ReLU_sim}, it can be concluded that $w_{min}(m, n, ELU)=w_{min}(m, n, SELU)=w_{min}(m, n, LeakyReLU)\leq \max(2m+1, n)$.
\end{proof}
Based on Theorem \ref{thm_m<n<=2m}, we can determine when the upper bound in Theorem \ref{thm_variant_LR_width} is attained.
\begin{theorem}[Conditions for Reaching the Upper Bounds]\label{thm_LRwidth}
	For the activation function $\sigma=LeakyReLU$, $LeakyReLU_\beta \left(\beta\in(0, 1)\cup (1, +\infty)\right)$, $ELU$, $SELU$, we have
	\begin{align*}
		w_{min}(m, 2m, \sigma)=2m+1.
	\end{align*}
\end{theorem}
\begin{proof}
	According to Theorem \ref{thm_injectivewidth}, we know that $w_{min}(m, 2m, LeakyReLU)\geq 2m+1$. Based on Hwang's results (\cite{Hwang}), we have $w_{min}(m, 2m, LeakyReLU)\leq 2m+1$. Therefore, we can show that $w_{min}(m, 2m, LeakyReLU)=w_{min}(m, 2m, LeakyReLU_\beta)=w_{min}(m, 2m, ELU)=w_{min}(m, 2m, SELU)=2m+1$ by Theorem \ref{thm_Leaky_equivalent} and Theorem \ref{thm_variant_ReLU_sim}.
\end{proof}
For certain activation functions, we show that the iterates of $\sigma$ can approximate $ReLU$ when they satisfy some conditions, from the geometric perspective of the image of one-dimensional input under ReLU, as stated in the following lemma.
\begin{lemma}[Approximation with Iterates of Activation Functions]\label{lem_iterate_approReLU}
	When the activation function $\sigma: \mathbb{R}\rightarrow \mathbb{R}$ satisfies
	\begin{enumerate}[(1)]
		\item $\sigma(x)=x$ holds for $x\geq 0$, and $\sigma(x)<0$ holds for $x<0$,
		\item $\forall c<0$, there exists $0\leq b<1$, such that for any $x\in (-\infty, c]$, we have $0\leq \frac{\sigma(x)}{x}\leq b<1$.
	\end{enumerate}
	Then for any fixed $x<0$, the iterated function satisfies that $\lim\limits_{n\rightarrow +\infty}\sigma^n(x)= 0$. Therefore, for any fixed $x\in \mathbb{R}$, $\lim\limits_{n\rightarrow +\infty}\sigma^n(x)=ReLU(x)$. 
\end{lemma}
\begin{proof}
	For any fixed $x<0$, we have $0> \sigma^n(x)=\sigma(\sigma^{n-1}(x))>\sigma^{n-1}(x)$, which implies the existence of the limit $a=\lim\limits_{n\rightarrow +\infty}\sigma^n(x)\leq 0$. Suppose, for the sake of contradiction, that $a=\lim\limits_{n\rightarrow +\infty}\sigma^n(x) <0$, then $\sigma^n(x)\leq a<0$ holds for all $n\in \mathbb{N}_+$. Since there exists $0\leq b<1$ such that $\sigma^n(x)\geq b\sigma^{n-1}(x)$ holds for all $n\in \mathbb{N}_+$, it follows that $\sigma^n(x)\geq b^n x$. Therefore, there exists $N\in \mathbb{N}_+$ such that $a<\sigma^n(x)<0$ holds for all $n>N$, which leads to a contradiction. Consequently, for any fixed $x\in\mathbb{R}$, we have $\lim\limits_{n\rightarrow +\infty}\sigma^n(x)= ReLU(x)$.
\end{proof}
In fact, there are a lot of activation funcions satisfy the conditions in Lemma \ref{lem_iterate_approReLU}, such as several variants of ReLU: $LeakyReLU_\beta\enspace(\beta\in (0, 1))$, $ELU_\beta\enspace(\beta\in(0, 1))$, $SELU_{(1, \beta)}\enspace(\beta\in(0, 1))$ (equivalent to ELU up to a multiplicative factor), etc. By Lemma \ref{lem_iterate_approReLU} and Lemma 23 in (\cite{Kim}), we establish the approximation relation for variants of ReLU.
\begin{theorem}[Approximation Relation for Varaints of ReLU]\label{thm_variant_ReLU_appro}
	For the activation function $\sigma_1=Softplus$, $CELU$, and $\sigma_2=LeakyReLU$, $ELU$, $SELU$, we have
	\begin{align*}
		N_{m, n, k}^{ReLU} \prec N_{m, n, k}^{\sigma_1} \prec N_{m, n, k}^{\sigma_2}.
	\end{align*}
\end{theorem}
\begin{proof}
	The proof is provided in Appendix \ref{appendix_variant_ReLU_appro}.
\end{proof}
Subsequently, the approximation relation for variants of ReLU immediately leads to the following result of minimum width.
\begin{theorem}[Minimum Width for Variants of ReLU]\label{thm_variant_ReLU_width}
	For the activation function $\sigma=LeakyReLU$, $ELU$,  $SELU$, $Softplus$, $CELU$, we have
	\begin{align*}
		m+1\leq w_{min}(m, n, \sigma)\leq m+n.
	\end{align*}
\end{theorem}
\begin{proof}
	The left inequality can be proved by the results in (\cite{Johnson}). The right inequality can be concluded by Theorem \ref{thm_variant_ReLU_appro} and the results in (\cite{Hanin}). 
\end{proof}
In practice, Theorem \ref{thm_variant_LR_width} and Theorem \ref{thm_variant_ReLU_width} provide guidance for determining the minimum width when we design a neural network. As an example, for the continuous functions $g_k$ ntroduced in the following proposition, we trained networks with a specified width to show that $g_k\prec N_{2, 2, 4}^{ELU}$.

\begin{proposition}[Application of the Minimum Width Properties]\label{prop_disk}
	Let $rot_k$ denotes the $k$-rotation map $rot_k$:
	\begin{align}
		rot_k(r\cos\theta,r\sin\theta) = (r\cos k\theta,r\sin k\theta) \label{eq_rotk},
	\end{align}
	then for $g_k=rot_k\vert_{B^2}$, where $B^2=\{(x, y)\big\vert x^2+y^2\leq 1\}$, $g_k\prec N_{2, 2, 4}^{ELU}$ for any $k\in \mathbb{N}^{+}$.
\end{proposition}
\begin{proof}
	$w_{min}(2, 2, ELU)\leq 4$ can be directly concluded by Theorem \ref{thm_variant_ReLU_width}. To further validate this minimum width result, we conducted numerical experiments, as shown in Subsection \ref{sec_4.2}.
\end{proof}
\section{Numerical Experiments}\label{sec_4}
Subsection \ref{sec_4.1} introduces the implementation details for training a neural network. The experiments in Subsection \ref{sec_4.2} aim to determine whether a given mapping can be approximated solely by increasing the network depth $d\in \mathbb{N}$ while keeping the network width $w\in \mathbb{N}$ fixed. For the convenience of network training, unless otherwise specified, the ELU activation function considered in this section is fixed to $ELU_1$.
\subsection{Implementation details}\label{sec_4.1}
In this subsection, we present the details of training the neural networks, including the datasets, loss functions, and the criteria for determining the termination and success of the experiments.
\subsubsection{Constructions of datasets and loss functions}
A dataset is defined as the collection of input-target pairs by convention (\cite{petersen_DL}), i.e.,
\begin{align*}
	\mathcal{D}=\{(x_i, y_i)\vert x_i\in \mathcal{A}, y_i\in \mathcal{B}, i\in\mathbb{N}_+\}.
\end{align*}
Here, $\mathcal{A}$ and $\mathcal{B}$ denote the input set and the target set, respectively. For DNNs, the output set $\mathcal{C}$ of a NN $\Phi_\sigma$ is defined by $\mathcal{C}=\{\Phi_\sigma(x)\vert x\in \mathcal{A}\}$.

The dataset $\mathcal{D}$ is partitioned into two subsets: the training dataset $\mathcal{D}_T$ and the validation dataset $\mathcal{D}_V$. The training loss $L$ is defined as the mean squared error between the network output and the target over the training dataset, while the validation loss $\widetilde{L}$ is the mean squared error on the validation dataset, which evaluates the model's performance on unseen data. For any function $f_\theta:\mathbb{R}^m \rightarrow \mathbb{R}^n$, the loss functions are given by
\begin{align*}
	L(f_\theta)=\frac{1}{N_1\times n} \sum_{(x_i, y_i)\in \mathcal{D}_T}\vert\vert y_i - f_\theta(x_i) \vert\vert_2^2,\\
	\widetilde{L}(f_\theta)=\frac{1}{N_2\times n} \sum_{(x_j, y_j)\in \mathcal{D}_V}\vert\vert y_j - f_\theta(x_j) \vert\vert_2^2,
\end{align*}
where $\theta$ represents the parameters controlling $f_\theta$ (weight matrices and bias vectors), and $N_1, N_2$ are the batch sizes of $\mathcal{D}_T$ and $\mathcal{D}_V$, respectively. The training loss $L$ and the validation loss $\widetilde{L}$ are key metrics used to monitor the model's performance and generalization ability. 
\subsubsection{Training process and criterion of successful training}
In the subsequent experiments, we restrict our attention to network models as defined in Definition \ref{defn_NN}. These models consist solely of fully connected neurons and exclude additional architectural components such as convolutional layers (\cite{CNN_dropout}) and skip connections (residual (\cite{Resnet})). Moreover, all hidden layers are designed to contain the same number of neurons, ensuring a uniform width across the network.

During training, networks attempt to minimize the training loss $L$ by adjusting their weights matrices $W_i$ and the bias vectors $b_j$ using the Adam optimizer (\cite{Adam}). The learning rate was initialized to $10^{-4}$ and decayed by $5\times 10^{-6}$ every 2 million steps. Training is terminated after at most 20 million steps. The success criterion is formally defined such that both $L$ and $\widetilde{L}$ must fall below the error threshold $10^{-4}$, thereby reducing the risk of overfitting.
\subsection{Experiments on the Approximation of $k$-rotation Map $rot_k$}\label{sec_4.2}
For the $k$-rotation map defined in Equation (\ref{eq_rotk}), we attempted to verify that $rot_k\vert_{B^2}\prec N_{2, 2, 4}^{ELU} (k\in \{2, 3, 4, 5, 6\})$, by training networks to approximate the map $rot_k\vert_{B^2}$. In addition, we conducted experiments to show that $rot_2\vert_{B^2}\nprec N_{2, 2, 3}^{ELU_1}$ by constraining the network width to 3.
\subsubsection{DISK Dataset}
The grid-based, discretely sampled 2-dimensional unit disk $D_1^2$ with its image under $k$-rotation map $rot_k(D_1^2)$ are combined to form the training dataset $\mathcal{D}_T$. For each $k\in \mathbb{N}$, they are defined as
\begin{align*}
	D_1^2&=\{(x,y)\vert x^2+y^2\leq 1, x=-1+0.02i, y=-1+0.02j, i, j\in \{0, 1, \cdots, 100\}\}, \\
	\mathcal{D}_T^{disk}&=\{(x, rot_k(x))\vert x\in D_1^2 \},
\end{align*}
where $rot_k$ is given by Equation (\ref{eq_rotk}). A distinct validation dataset $\mathcal{D}_V$ is constructed for each $k\in \mathbb{N}$, ensuring no overlap with the training points, i.e., 
\begin{align*}
	D_2^2&=\{(x,y)\vert x^2+y^2\leq 1, x=-1+0.05i, y=-1+0.05j, i, j\in \{0, 1, \cdots, 40\}\}\backslash D_1^2,\\
	\mathcal{D}_V^{disk}&=\{(x, rot_k(x))\vert x\in D_2^2 \}.
\end{align*}
Subsequently, the dataset denoted by $DISK$, is given by $DISK = \mathcal{D}_T^{disk} \cup \mathcal{D}_V^{disk}.$
\subsubsection{Experimental Results}
To investigate the approximation ability of ELU NNs, two experiments were conducted. 

In the first experiment, the network width was fixed at 4. The results indicate that a ELU NN with $ELU_1$ activation function, width 4 and depth 4 could successfully approximate the 2-rotation map $rot_2:\mathbb{R}^2\rightarrow \mathbb{R}^2$. The target and output sets of this experiment are depicted in Figure \ref{fig_disk_w4}, and the parameters for this ELU NN are provided in Appendix \ref{appendix_diskNN}. To investigate the relationship between the minimum depth $d_{min}$ and the $k$-rotation map $rot_k:\mathbb{R}^2\rightarrow \mathbb{R}^2$, the width was fixed at $w=4$, and the depth $d$ was gradually increased until the output set of the network sucessfully approximated the target set for specfic $k\in\mathbb{N}$, then we got the minimum depth $d_{min}$. The detailed results are presented in Table \ref{table_disk_w4}.
\begin{figure}[htbp]
	\centering
	\subfigure[Target set]{
		\includegraphics[width=0.4\linewidth, height=0.35\linewidth]{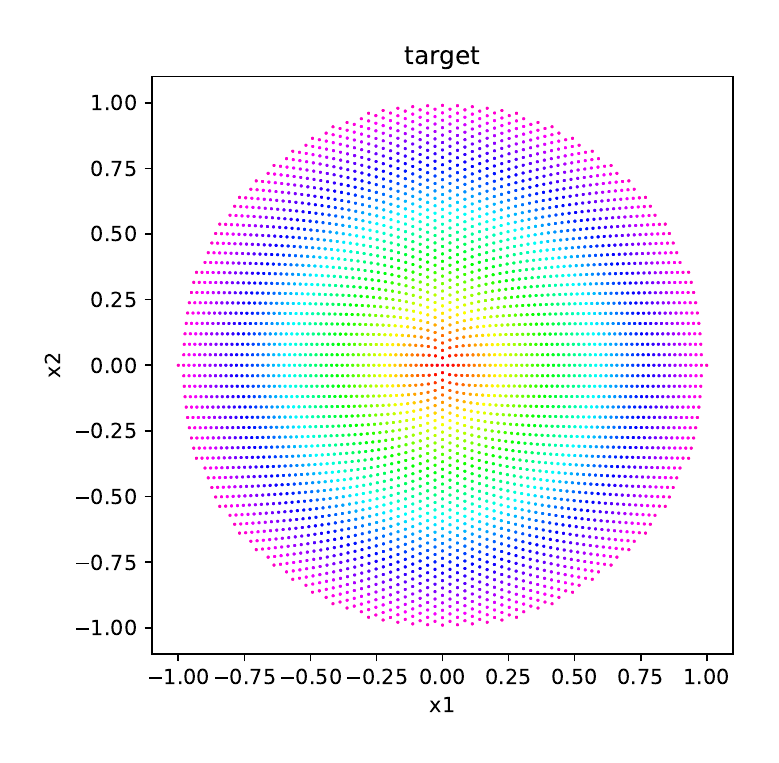}
	}
	\subfigure[Output set]{
		\includegraphics[width=0.4\linewidth, height=0.35\linewidth]{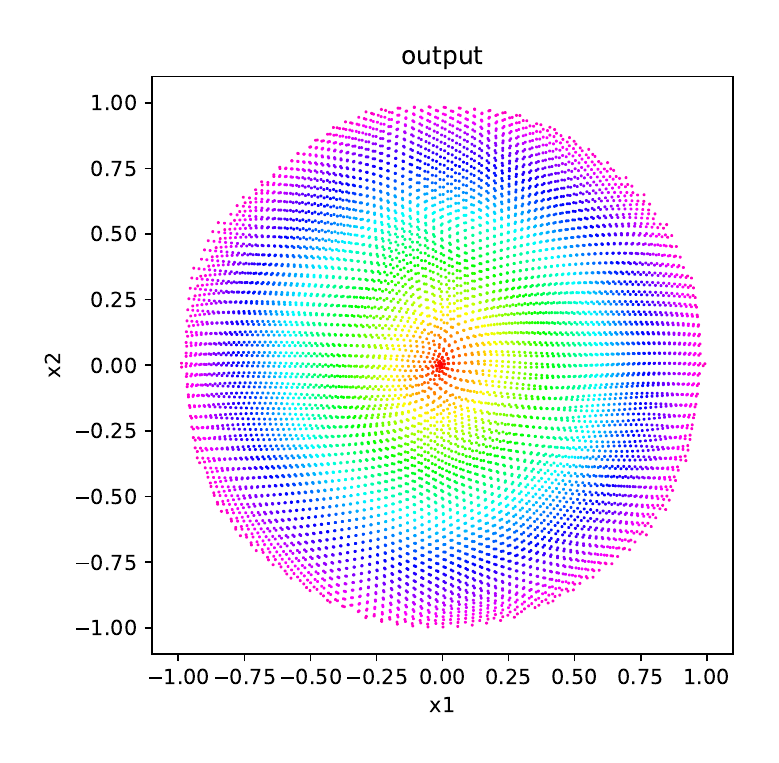}
	}
	\caption{Comparison between the target and output sets when using an ELU network with width 4 and depth 4 to approximate the 2-rotation map $rot_2$}
	\label{fig_disk_w4}
\end{figure} 
\begin{table}[htbp]
	\centering
	\caption{A brief description of the relation between rotation parameter $k$ and minimum depth $d_{min}$. All the networks used the $ELU_1$ activation function, and width $w$ was fixed to 4. $L$ and $\widetilde{L}$ represent the training loss and validation loss, respectively.}
	\label{table_disk_w4}
	\begin{tabular}{c | c |c c c c}
		\hline
		\multirow{2}{*}{$\sigma$} & \multirow{2}{*}{$w$} & \multirow{2}{*}{$k$} & \multirow{2}{*}{$d_{min}$} & \multirow{2}{*}{$L$} & \multirow{2}{*}{$\widetilde{L}$}\\
		& & & & & \\
		\hline
		\multirow{5}{*}{$ELU_1$} & \multirow{5}{*}{4} & 2 & 4 & $6.03\times 10^{-5}$ & $5.99\times 10^{-5}$\\
		\cline{3-6}
		& & 3 & 6 & $7.98\times 10^{-5}$ & $7.94\times 10^{-5}$\\
		\cline{3-6}
		& & 4 & 8 & $8.08\times 10^{-5}$ & $8.02\times 10^{-5}$\\
		\cline{3-6}
		& & 5 & 10 & $9.48\times 10^{-5}$ & $9.72\times 10^{-5}$\\
		\cline{3-6}
		& & 6 & 12 & $9.20\times 10^{-5}$ & $8.46\times 10^{-5}$\\
		\hline
	\end{tabular}
\end{table}

To show the impact of network width, we conducted a comparative experiment. The width was fixed at 3 in the second experiment, while the depth $d$ was varied from 1 to 30 to approximate the 2-rotation map $rot_2:\mathbb{R}^2\rightarrow \mathbb{R}^2$. Figure \ref{fig_disk_w3_loss} illustrates the relationship between the training loss $L$ and the depth $d$, with the dashed line denoting the error threshold. It is shown that only when the depth reached 12 did the training loss $L$ fell below the error threshold of $10^{-4}$. However, the corresponding validation loss remained above the error threshold, so we needed to further examine this case. Figure \ref{fig_disk_w3} shows the target and output sets of this experiment, for the case that $w=3$ and $d=12$. The comparison between the target and output sets reveals that the output set exhibited a hole, as shown in the black box in Figure \ref{fig_disk_w3}, whereas the target set did not not have such defects. With the analysis of loss functions and the output of the network, it is indicated that this network suffered from overfitting. In conclusion, our results indicate that the networks with width 3 and $ELU_1$ activation function can not accurately approximate the given mapping, regardless of depth. This experiment showed that $rot_2\nprec N_{2, 2, 3}^{ELU_1}$, and presented the essential role of adequate network width in function approximation.
\begin{figure}[htbp]
	\centering
	\subfigure[Training loss $L$]{
		\includegraphics[width=0.4\linewidth, height=0.4\linewidth]{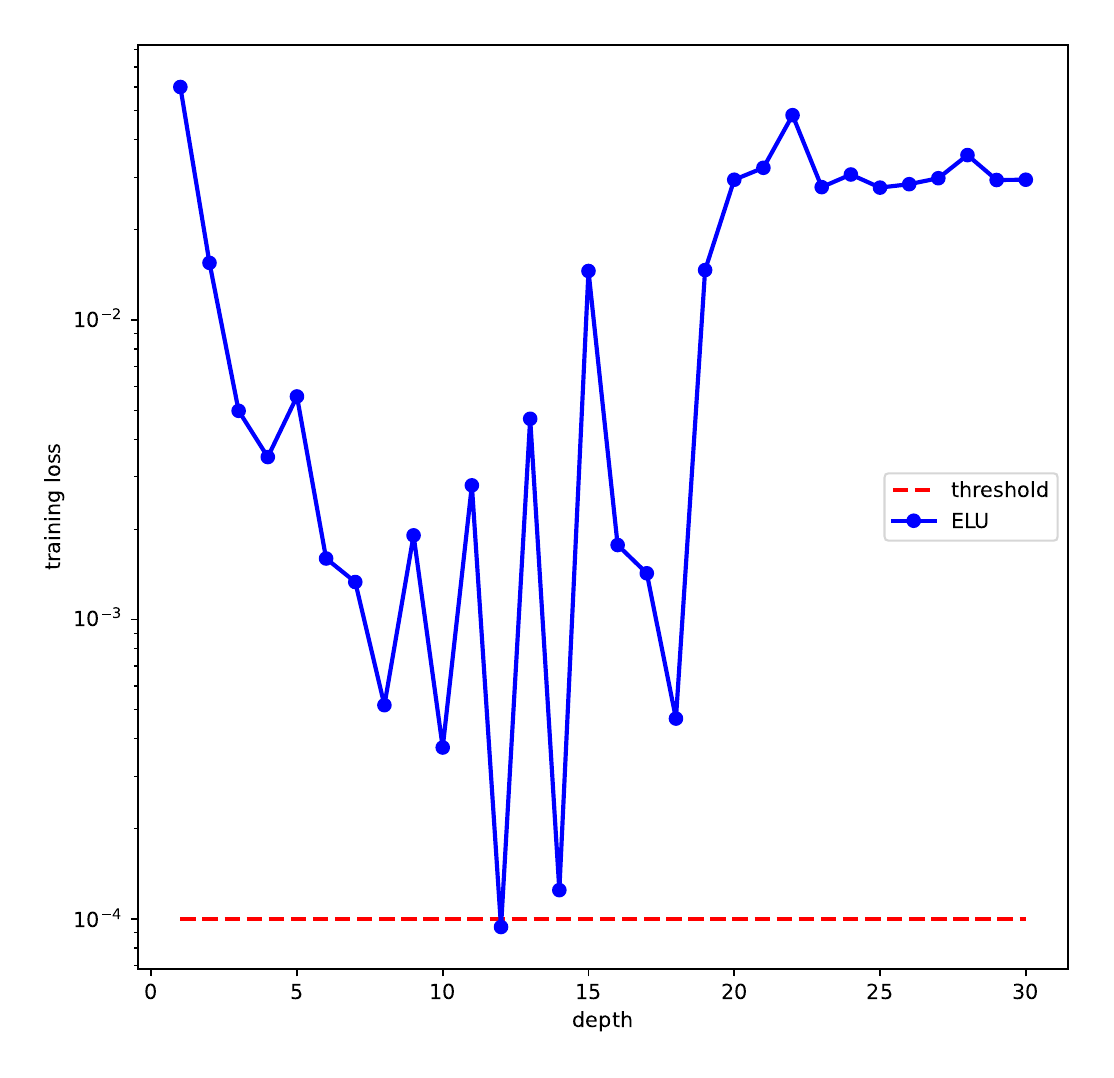}
	}
	\subfigure[Validation loss $\widetilde{L}$]{
		\includegraphics[width=0.4\linewidth, height=0.4\linewidth]{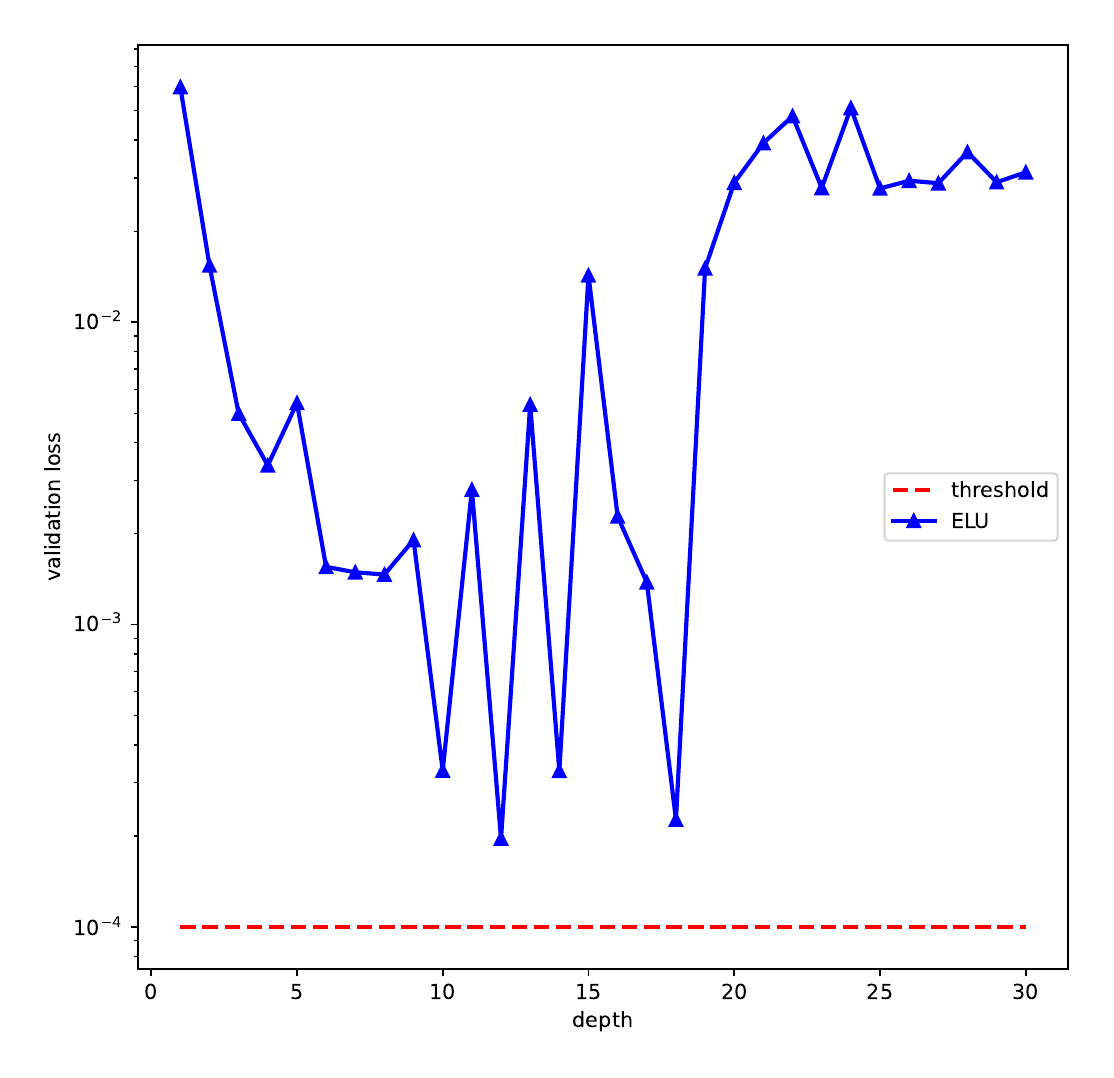}
	}
	\caption{The digram of the loss functions for the experiment on the DISK dataset.}
	\label{fig_disk_w3_loss}
\end{figure} 
\begin{figure}[htbp]
	\centering
	\subfigure[Target set]{
		\includegraphics[width=0.35\linewidth, height=0.35\linewidth]{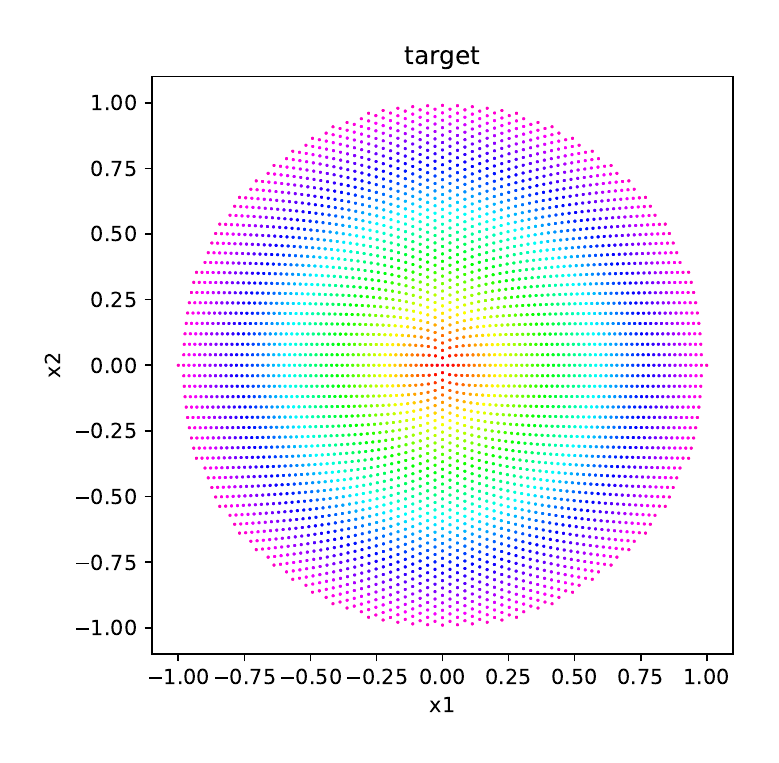}
	}
	\subfigure[Output set]{
		\includegraphics[width=0.51\linewidth, height=0.35\linewidth]{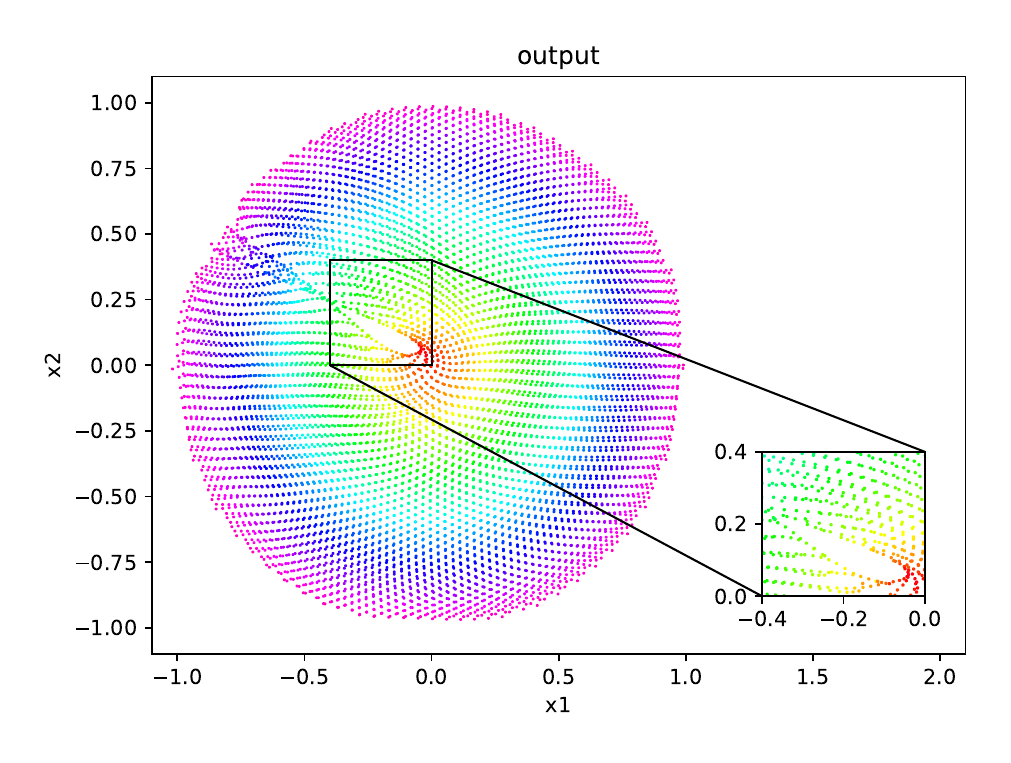}
	}	
	\caption{The case that the network width and depth were 3 and 12, respectively. The black box shows the hole in the output.}
	\label{fig_disk_w3}
\end{figure} 

\section{Discussion}\label{sec_5}
In light of the above discussions, an elegant fact about the DNNs is that the upper and lower bounds of the minimum width $w_{min}$ depend on multiple factors, including the choice of activation functions, the geometric constraints (e.g., self-intersection) and the input-output dimensions. The progress made so far bring more hope to deeper investigation of deep narrow networks. In particular, we are looking forward to further refine the minimum width $w_{min}$ for a broader class of activation functions under specific dimensional conditions. In the future, we will explore how narrow networks can be trained more effectively from the perspective of experimental feasibility, for instance, by adjusting network initialization (\cite{initialization_narrowNN}).

Moreover, the emerging field of geometric and topological approaches for neural networks (e.g., persistent homology (\cite{PH_NN}), discrete curvature (\cite{Ricci_cur})) opens up additional exciting new directions. For instance, in this paper, the Poincar\'e-Miranda Theorem serves as a key analytical tool, enabling us to quantitatively determine when self-intersections occur and thereby violate the geometric properties of networks. For future work, we anticipate that such techniques will provide deeper theoretical insights into the essential characteristics of DNNs.

\acks{
	Funding in direct support of this work: GPUs donated by the HPC Platform of Huazhong University of Science and Technology., scholarship by Huazhong University of Science and Technology (Wuhan, 430074, China) and the Hubei Key Laboratory of Engineering Modeling and Scientific Computing, Huazhong University of Science and Technology (Wuhan 430074, China)}

\newpage

\appendix
	\section{Supplements}
\subsection{Activation Function Formulae}\label{appendix_activationfunc}
The common activation functions are defined as follows (\cite{petersen_DL}; \cite{Kim}).
\begin{itemize}
	\item ReLU:	
	$$
	\scriptsize{ReLU(x)=\left\{
		\begin{aligned}
			x,&\enspace x\geq0,\\
			0,&\enspace x<0. 
		\end{aligned}
		\right.}
	$$
	\item LeakyReLU: for $\beta\in (0, +\infty)$,
	$$
	\scriptsize{LeakyReLU_\beta(x)=\left\{
		\begin{aligned}
			x,&\enspace x\geq0,\\
			\beta x,&\enspace x<0. 
		\end{aligned}
		\right.}
	$$
	\item Exponential Linear Unit (ELU): for $\beta\in(0, +\infty)$,
	$$
	\scriptsize{ELU_{\beta}(x)=\left\{
		\begin{aligned}
			x,&\enspace x\geq0,\\
			\beta(e^x-1),&\enspace x<0.
		\end{aligned}
		\right.}
	$$
	\item Continuously differentiable Exponential Linear Unit (CELU): for $\beta\in (0, +\infty)$,
	$$
	\scriptsize{CELU_{\beta}(x)=\left\{
		\begin{aligned}
			x,&\enspace x\geq0,\\
			\beta(e^{(\frac{x}{\beta})}-1),&\enspace x<0.
		\end{aligned}
		\right.}
	$$
	\item Scaled Exponential Linear Unit (SELU): for $\lambda\in (0, +\infty)$ and $\beta\in (0, +\infty)$,
	$$
	\scriptsize{SELU_{(\lambda, \beta)}(x)=
		\lambda\times\left\{
		\begin{aligned}
			x,&\enspace x\geq0,\\
			\beta(e^{x}-1),&\enspace x<0.
		\end{aligned}
		\right.}
	$$
	\item Softplus: for $\beta\in (0, +\infty)$,
	$$
	\scriptsize{Softplus_{\beta}(x)=\frac{1}{\beta}\ln(1+e^{\beta x}).}
	$$
	\item HardTanh:
		$$
	\scriptsize{HardTanh(x)=\left\{
		\begin{aligned}
			1,&\enspace x>1,\\
			x, &\enspace -1\leq x\leq 1,\\
			-1,&\enspace x<-1.
		\end{aligned}
		\right.}
	$$
	\item ReLU6:
	$$
	\scriptsize{ReLU6(x)=\left\{
		\begin{aligned}
			6,&\enspace x\geq6,\\
			x,&\enspace 0\leq x\leq6,\\
			0,&\enspace x<0. 
		\end{aligned}
		\right.}
	$$
\end{itemize}

\section{Proofs}

\subsection{Proof of Lemma \ref{lem_appro_indimn}}\label{appendix_appro_indimn}
Denote $D=[a_{11}, a_{12}]\times [b_{11}, b_{12}]\times \cdots\times [a_{m1}, a_{m2}]\times[b_{m1}, b_{m2}]$, $x=(x_1, x_2, \cdots, x_m)$, $y=(y_1, y_2, \cdots, y_m)$. Let $h: D\rightarrow \mathbb{R}^{2m}$, $h=(h_1, h_2, \cdots, h_{2m})$ be defined by:
$$
\begin{aligned}
	h(x_1, y_1, x_2, y_2, \cdots, x_m, y_m)= f(x) - f(y),\enspace (x, y)\in D.
\end{aligned}
$$
We can define
$$
\scriptsize{
\begin{aligned}
M= \max\Big\{
\max_{k\in\{1, 2, \cdots, m\}}
\sup_{(x, y)\in D} [g_{2k-1}(x_1, \cdots, a_{k1}, \cdots, x_m) - g_{2k-1}(y)]*
	[g_{2k-1}(x_1, \cdots, a_{k2}, \cdots, x_m)-g_{2k-1}(y)],\\
	\max_{k\in\{1, 2, \cdots, m\}}
\sup_{(x, y)\in D} [g_{2k}(x) - g_{2k}(y_1, \cdots, b_{k1}, \cdots, y_m)]*
	[g_{2k}(x)-g_{2k}(y_1, \cdots, b_{k2},\cdots, y_m)]
\Big\}.
\end{aligned}
}
$$
By the product conditions, we have $M<0$. Then we choose a sufficiently small $\varepsilon>0$ satisfies that
$$
\scriptsize{
	\begin{aligned}
	\varepsilon<\min
	\Bigg\{
	\frac{1}{2}, \min_{k\in\{1, 2,\cdots, m\}} \frac{-M}{2\sup\limits_{(x, y)\in D}\Big(
		\vert g_{2k-1}(x_1, \cdots, a_{k1}, \cdots, x_m)- g_{2k-1}(y)\vert + \vert g_{2k-1}(x_1, \cdots, a_{k2}, \cdots, x_m)-g_{2k-1}(y) \vert  \Big)+1},\\
		\min_{k\in\{1, 2,\cdots, m\}} \frac{-M}{2\sup\limits_{(x, y)\in D}\Big(
			\vert g_{2k}(x)- g_{2k}(y_1, \cdots, y_k, \cdots, y_m)\vert + \vert g_{2k}(x)-g_{2k}(y_1, \cdots, y_k, \cdots, y_m) \vert  \Big)+1}
	 \Bigg\}.
\end{aligned}
}
$$
When $\sup_{t\in [0, 1]^m}\vert\vert f(t)-g(t)\vert \vert_{\infty}<\varepsilon$, it holds that
$$\scriptsize{
	\begin{aligned}
	&h_{2k-1}(x_1, y_1, \cdots, a_{k1}, y_k, \cdots, x_m, y_m)*h_{2k-1}(x_1, y_1, \cdots, a_{k2}, y_k, \cdots, x_m, y_m)\\
&=\left[g_{2k-1}(x_1, \cdots, a_{k1}, \cdots, x_m) -g(y)\right]*
\left[g_{2k-1}(x_1, \cdots, a_{k2}, \cdots, x_m)-g(y)
\right] +\\ 
&\left[ f_{2k-1}(x_1, \cdots, a_{k1}, \cdots, x_m)-f_1(y)-g_{2k-1}(x_1, \cdots, a_{k1}, \cdots, x_m) + g_{2k-1}(y)
\right]*
\left[g_{2k-1}(x_1, \cdots, a_{k2}, \cdots, x_m)-g_{2k-1}(y)
\right]+\\ 
&\left[ f_{2k-1}(x_1, \cdots, a_{k2}, \cdots, x_m)-f_1(y)-g_{2k-1}(x_1, \cdots, a_{k2}, \cdots, x_m) + g_{2k-1}(y)
\right]*
\left[g_{2k-1}(x_1, \cdots, a_{k1}, \cdots, x_m)-g_{2k-1}(y)
\right]+\\
&\left[ f_{2k-1}(x_1, \cdots, a_{k1}, \cdots, x_m)-f_1(y)-g_{2k-1}(x_1, \cdots, a_{k1}, \cdots, x_m) + g_{2k-1}(y)
\right]*\\
&\left[ f_{2k-1}(x_1, \cdots, a_{k2}, \cdots, x_m)-f_1(y)-g_{2k-1}(x_1, \cdots, a_{k2}, \cdots, x_m) + g_{2k-1}(y)
\right]\\
&<M+2\varepsilon\vert g_{2k-1}(x_1, \cdots, a_{k1}, \cdots, x_m)-g_{2k-1}(y)\vert +
2\varepsilon\vert g_{2k-1}(x_1, \cdots, a_{k2}, \cdots, x_m)-g_{2k-1}(y)\vert+4\varepsilon^2\\
&<M+\Big(2\sup_{(x, y)\in D} \left(\vert g_{2k-1}(x_1, \cdots, a_{k1}, \cdots, a_m)-g_{2k-1}(y)\vert +\vert g_{2k-1}(x_1, \cdots, a_{k2}, \cdots, x_m)-g_{2k-1}(y)\vert\right)+1\Big)\varepsilon\\
&<0.
\end{aligned}
}
$$
Then we can similarly prove that 
$$
\begin{aligned}
	h_{2k}(x_1, y_1, \cdots, x_k, b_{k1}, \cdots, x_m, y_m)*h_{2k}(x_1, y_1, \cdots, x_k, b_{k2}, \cdots, x_m, y_m)<0.
\end{aligned}
$$
Therefore, by Poincare-Miranda Theorem, there exists a $t\in D$ such that $h(t)=0$. Hence, we can conclude that the map $f:[0, 1]^m\rightarrow \mathbb{R}^{2m}$ can not be injective.

\subsection{Proof of Theorem \ref{thm_Leaky_equivalent}}\label{appendix_Leaky_equivalent}
For a fixed parameter $\beta\in (0, 1)\cup (1, +\infty)$, Corollary 11 in (\cite{Hwang}) shows that $N_{m, n, k}^{LeakyReLU_\beta}\prec N_{m, n, k}^{LeakyReLU}$. 

Since 
\begin{align*}
	LeakyReLU_\frac{1}{\alpha}=T_{-\frac{1}{\alpha} I_{k\times k}, 0_{k\times k}}\circ LeakyReLU_\alpha \circ T_{-I_{k\times k}, 0_{k\times k}},
\end{align*} 
where $I_{k\times k}$ and $0_{k\times k}$ are $k\times k$ identity matrix and $k\times k$ zero matrix, respectively. It is evident that $N_{m, n, k}^{LeakyReLU_\alpha}=N_{m, n, k}^{LeakyReLU_{\frac{1}{\alpha}}}$, so we only need to consider the cases with $\alpha\in (0, 1]$. 

To prove the reverse relation $N_{m, n, k}^{LeakyReLU}\prec N_{m, n, k}^{LeakyReLU_\beta}$, for $\Phi_{LeakyReLU}\in N_{m, n, k}^{LeakyReLU}$ represented by
\begin{align*}
	\Phi_{LeakyReLU}= T_{W_L, b_L} \circ LeakyReLU_{\alpha_L} \circ \cdots \circ LeakyReLU_{\alpha_1} \circ T_{W_0, b_0},
\end{align*}
we can prove that $LeakyReLU_\alpha \prec N_{1, 1, 1}^{LeakyReLU_\beta}$ for any $\alpha\in(0, 1]$. Then for any $\varepsilon>0$ and a compact set $K$, we can substitute $LeakyReLU_{\alpha_i}$ with a $g_{LR_\beta}^{(i)}\in N_{k, k, k}^{LeakyReLU_\beta}$ represented by
\begin{align*}
	g_{LR_\beta}^{(i)} &=T_{\lambda_j^{(i)} I_{k\times k}, b_j^{(i)} I_k}\circ LeakyReLU_\beta \circ  \cdots \circ LeakyReLU_\beta \circ T_{\lambda_0^{(i)} I_{k\times k}, b_0^{(i)} I_k},
\end{align*}
such that
\begin{align*}
	\vert\vert \Phi_{LeakyReLU}(x) - T_{W_L, b_L} \circ g_{LR_\beta}^{(L)} \circ \cdots \circ g_{LR_\beta}^{(1)}\circ T_{W_0, b_0}(x)\vert\vert_{\infty}<\varepsilon
\end{align*}
holds for any $x\in K$, where $I_k$ in $g_{LR_\beta}^{(i)}$ is the $k$-dimensional column vector whose elements are all 1.

To prove that $LeakyReLU_\alpha \prec N_{1, 1, 1}^{LeakyReLU_\beta} \enspace (\alpha\in(0, 1])$, when $\alpha=1$, $\Phi_1(x)=T_{1, 0}(x)\in N_{1, 1, 1}^{LeakyReLU_\beta}$ equals to $LeakyReLU_1$. When $\alpha\in(0, 1)$, for any $n\in\mathbb{N}_+$, since we have
 \begin{align*}
	LeakyReLU_{\beta^n}&=T_{1, 0}\circ LeakyReLU_\beta\circ \cdots \circ LeakyReLU_\beta \circ T_{1, 0},\\
	LeakyReLU_{\beta^{-n}}&=(T_{-\frac{1}{\beta}, 0}\circ LeakyReLU_\beta \circ T_{-1, 0})\circ \cdots \circ  (T_{-\frac{1}{\beta}, 0}\circ LeakyReLU_\beta \circ T_{-1, 0}),
\end{align*}
such that for any $k\in \mathbb{Z}$, $LeakyReLU_{\beta^k}$ can be expressed by a NN activated by $LeakyReLU_\beta$. Without loss of generality, we assume that $\alpha\notin\{\beta^k\vert k\in \mathbb{Z}\}$. It is easy to see that there exist $\beta_1, \beta_2\in \{\beta^k\vert k\in \mathbb{Z}\}$, such that $0<\beta_1<\alpha<\beta_2$. For any $\varepsilon>0$, the sequence of LeakyReLU NNs $\{\Phi_i\vert i\in \mathbb{N}_+\}$ can be defined as follows:
\begin{align*}
	\Phi_1(x)&=T_{1, 0}\circ LeakyReLU_{\beta_2} \circ T_{1, 0}(x),\\
	\Phi_2(x)&=T_{1, -b} \circ LeakyReLU_{\frac{\beta_1}{\beta_2}} \circ T_{1, b}\circ \Phi_1(x),\\
	&\vdots\\
	\Phi_{2k-1}(x)&=T_{1, -(k-1)\varepsilon} \circ LeakyReLU_{\frac{\beta_2}{\beta_1}} \circ T_{1, (k-1)\varepsilon} \circ \Phi_{2k-2}(x),\\
	\Phi_{2k}(x)&=T_{1, -(k-1)\varepsilon-b}\circ LeakyReLU_{\frac{\beta_1}{\beta_2}}\circ T_{1, (k-1)\varepsilon+b}\circ \Phi_{2k-1}(x),
\end{align*}
where $b$ is a constant 
\begin{align*}
	b=\frac{\beta_1^{-1}-\alpha^{-1}}{\beta_1^{-1}-\beta_2^{-1}} * \varepsilon.
\end{align*}
We can verify the approximation property that $\vert\Phi_{2k}(x)-LeakyReLU_\alpha\vert \leq \varepsilon$ holds for any $x\in [-\frac{k\varepsilon}{\alpha}, +\infty)$.

Consequently, for any compact set $K\subset \mathbb{R}$, any $\varepsilon>0$ and a LeakyReLU activation function $LeakyReLU_\alpha$, there exists a NN $\Phi(x)$ activated by $LeakyReLU_\beta$ such that $\vert LeakyReLU_\alpha(x)-\Phi(x)\vert <\varepsilon$, which implies that $N_{m, n, k}^{LeakyReLU}\prec N_{m, n, k}^{LeakyReLU_\beta}$.

\subsection{Proof of Theorem \ref{thm_variant_ReLU_sim}}\label{appendix_variant_ReLU_sim}
Corollary 11 in (\cite{Hwang}) implies that $N_{m, n, k}^{\sigma}\prec N_{m, n, k}^{LeakyReLU}$ for any continuous increasing activation function $\sigma$. We now show the property $N_{m, n, k}^{LeakyReLU}\prec N_{m, n, k}^{ELU}$, which implies the relation $N_{m, n, k}^{ELU}\sim N_{m, n, k}^{LeakyReLU}$.

To prove this result, for $\Phi_{LeakyReLU}\in N_{m, n, k}^{LeakyReLU}$ represented by
\begin{align*}
	\Phi_{LeakyReLU}= T_{W_L, b_L} \circ LeakyReLU_{\alpha_L} \circ \cdots \circ LeakyReLU_{\alpha_1} \circ T_{W_0, b_0},
\end{align*}
we can prove that $LeakyReLU_\alpha\prec N_{1, 1, 1}^{ELU}$ for any $\alpha>0$. Then for any $\varepsilon>0$ and compact set $K$, we can substitute $LeakyReLU_{\alpha_i}$ with a $g_{ELU}^{(i)}\in N_{k, k, k}^{ELU}$ represented by
\begin{align*}
	g_{ELU}^{(i)} &=T_{\lambda_j^{(i)} I_{k\times k}, b_j^{(i)} I_k} \circ ELU_j \circ  \cdots \circ ELU_1 \circ T_{\lambda_0^{(i)} I_{k\times k}, b_0^{(i)} I_k},
\end{align*}
such that
\begin{align*}
	\vert\vert \Phi_{LeakyReLU}(x) - T_{W_L, b_L} \circ g_{ELU}^{(L)} \circ \cdots \circ g_{ELU}^{(1)}\circ T_{W_0, b_0}(x)\vert\vert_{\infty}<\varepsilon
\end{align*}
holds for any $x\in K$, where $I_{k\times k}$, $I_k$ and $ELU_j$ in $g_{ELU}^{(i)}$ are $k\times k$ identity matrix, $k$-dimensional column vector whose elements are all 1 and ELU activation function, respectively.

To prove that $LeakyReLU_\alpha\prec N_{1, 1, 1}^{ELU}\enspace (\alpha>0)$, for any compact set $K=[a, b]$ with $a<0<b$, we can construct a sequence of ELU NNs $\{\Phi_i\vert i\in \mathbb{N}_+\}$, such that for any $\varepsilon>0$ and fixed $\alpha>0$, there exists a $\Phi_N\in \{\Phi_i\vert i\in \mathbb{N}_+\}$ such that $\vert \Phi_N(x)-LeakyReLU_\alpha(x) \vert<\varepsilon \enspace(\forall x\in [a, b])$. The networks $\Phi_i(x)\enspace(i\in \mathbb{N}_+)$ are defined as follows:
\begin{align*}
   \Phi_1(x)&=T_{1, 0}\circ ELU_{k_1} \circ T_{1, 0}, \enspace k_1=\frac{-\varepsilon}{e^{\frac{-\varepsilon}{\alpha}}-1},\\
   \Phi_2(x)&=T_{1, -\varepsilon} \circ ELU_{k_2}\circ T_{1, \varepsilon}\circ \Phi_1(x), \enspace k_2=\frac{-\varepsilon}{e^{\Phi_1(\frac{-2\varepsilon}{\alpha})+\varepsilon}-1},\\
   &\vdots\\
   \Phi_n(x)&=T_{1, -(n-1)\varepsilon} \circ ELU_{k_n} \circ T_{1, (n-1)\varepsilon}\circ \Phi_{n-1}(x), \enspace k_n=\frac{-\varepsilon}{e^{\Phi_{n-1}(\frac{-n\varepsilon}{\alpha})+(n-1)\varepsilon}-1}.
\end{align*}
We can find that $\Phi_k$ and $LeakyReLU_{\alpha}$ strictly increase in $\mathbb{R}$, and $\Phi_{k}^{\prime}>0$ holds for every interval $(\frac{-i\varepsilon}{\alpha}, \frac{-(i-1)\varepsilon}{\alpha})\enspace(i\in \{1, 2, \cdots, k\})$, and we have
\begin{align*}
	\Phi_k\left(\frac{-i\varepsilon}{\alpha}\right)=LeakyReLU_\alpha\left(\frac{-i\varepsilon}{\alpha}\right) \enspace (\forall i \in \{1, 2, \cdots, k\}),
\end{align*}
then for every $x\in [\frac{-k \varepsilon}{\alpha}, 0]$, we have
\begin{align*}
    0 \leq LeakyReLU_{\alpha}(x)-\Phi_{k}(x)\leq LeakyReLU_{\alpha}\left(\frac{-i \varepsilon}{\alpha}\right)-\Phi_{k}\left(\frac{-(i+1)\varepsilon}{\alpha}\right)\leq\varepsilon.
\end{align*}
Since $LeakyReLU_\alpha(x)=\Phi_k(x)$ holds for every $x\in [0, +\infty)$, for any compact set $K\subset \mathbb{R}$, we can choose a ELU NN $\Phi_N$ such that the approximation property $\vert \Phi_N(x)-LeakyReLU_\alpha(x)\vert \leq \varepsilon$ holds for every $x\in K$. Therefore, we have $LeakyReLU_\alpha\prec N_{1, 1, 1}^{ELU}$, which leads to the property $N^{LeakyReLU}_{m, n, k}\prec N^{ELU}_{m, n, k}$.

For SELU activation function, since $SELU_{(1, \alpha)}=ELU_{\alpha}$, we have $ELU\subset SELU$, which immdeiately leads to the property $N^{LeakyReLU}_{m, n, k}\prec N^{SELU}_{m, n, k}$. Consequently, we have $N_{m, n, k}^{LeakyReLU}\sim N_{m, n, k}^{SELU}$.
\subsection{Proof of Theorem \ref{thm_variant_ReLU_appro}}\label{appendix_variant_ReLU_appro}
Suppose that the NN with ReLU is represented as
\begin{align*}
	\Phi_{ReLU} &= T_{W_L, b_L} \circ ReLU \circ \cdots \circ ReLU \circ T_{W_0, b_0},
\end{align*}
we can prove that $N_{m, n, k}^{ReLU} \prec N_{m, n, k}^{\sigma}$ when $\sigma=CELU, Softplus$.

For $\sigma=CELU_\beta\enspace(\beta\in (0, +\infty))$, by Lemma \ref{lem_iterate_approReLU}, we can prove that for any compact domain $K$ and $\varepsilon>0$, there exists $\delta>0$ and $N\in \mathbb{N}_+$, the $g_\sigma$ which is obtained by 
\begin{align*}
	g_\sigma &=T_{I_{k\times k}, 0_{k\times k}} \circ \sigma \circ \cdots \circ \sigma \circ T_{I_{k \times k}, 0_{k\times k}}
\end{align*}
satisfies the uniform approximation property $\vert \vert ReLU(x) - g_\sigma(x)\vert\vert_\infty<\varepsilon \enspace (x\in K)$, where $I_{k\times k}$ and $0_{k\times k}$ are the $k\times k$ identity matrix and $k\times k$ zero matrix, respectively. 

For the case that $\sigma=Softplus_\beta\enspace(\beta\in(0, +\infty))$, we consider $f_n$:
\begin{align*}
	f_n(x)=\frac{Softplus_\beta(nx)}{n}=\frac{\ln(1+e^{n\beta x})}{nx}.
\end{align*}
When $x\geq 0$, by the Mean Value Theorem, we have $0\leq f_n(x)-ReLU(x)\leq \frac{1}{n \beta}$. When $x<0$, we have $0\leq f_n(x)-ReLU(x)=\frac{\ln(1+e^{n\beta x})}{n\beta}\leq \frac{2}{n\beta}$. Therefore, for any $\varepsilon>0$ and compact domain $K$, there exists $N\in\mathbb{N}_+$, then the $g_\sigma$ which is obtained by
\begin{align*}
	g_\sigma=T_{\frac{1}{N} I_{k\times k}, 0_{k\times k}} \circ Softplus_\beta \circ T_{N I_{k\times k}, 0_{k\times k}}
\end{align*}
satisfies that $\vert\vert ReLU(x)-g_\sigma\vert\vert_{\infty}<\varepsilon\enspace (x\in K)$.

Through the substitution of ReLU in $\Phi_{ReLU}$ with $g_\sigma$, for any $\varepsilon>0$ and compact set $K$, we can construct a new NN $\Phi_\sigma$ activated by $\sigma$, which is represented as
\begin{align*}
	\Phi_\sigma=T_{W_L, b_L}\circ g_\sigma \circ \cdots \circ g_\sigma \circ T_{W_0, b_0},
\end{align*}
such that $\vert\vert\Phi_{ReLU}-\Phi_\sigma(x)\vert\vert_{\infty}<\varepsilon \enspace (x\in K)$, when the network depth is suffcient.

To prove that $N_{m, n, k}^{\sigma} \prec N_{m, n, k}^{LeakyReLU}$ holds for $\sigma=CELU, Softplus$, we can accomplish this proof by the Corollary 11 of (\cite{Hwang}). Theorem \ref{thm_variant_ReLU_sim} implies that $N_{m, n, k}^{\sigma} \prec N_{m, n, k}^{ELU}$ and $N_{m, n, k}^{\sigma} \prec N_{m, n, k}^{SELU}$, then we finish the proof.

\section{Parameter Settings of Networks}\label{appendix_diskNN}
In this section, the parameters for the network in Fig \ref{fig_disk_w4} are presented. The weight matrices $W_i$ and bias vectors $b_j$ are defined according to Equation (\ref{eq_xk}). Elements in matrices are retained to four decimal places when their absolute values exceed $10^{-4}$.

The NN is activated by $ELU_1$, the weight matrices $\{W_i\}$ and bias vectors $\{b_i\}$ are as follows.
$$
\begin{array}{ll}
	&\scriptsize{
		W_0=\left(
		\begin{array}{cc}
			0.2242 & -1.4544\\
			-0.2940 & -1.1861\\
			-1.0563 & -0.8876\\
			-0.4160 & 0.2796
		\end{array}
		\right),
		W_1=\left(
		\begin{array}{cccc}
			2.2738 & -1.9410 & 1.6917 & 2.0169\\
			1.2273 & -1.0533 & 0.6385 & -1.4016\\
			-0.6659 &  0.6353 & 0.2787 & 0.3609\\
			0.9003 & -1.4725 & 0.3804 & 0.2183
		\end{array}
		\right),
	}\\
	&\scriptsize{
		W_2=\left(
		\begin{array}{cccc}
			-0.0287 & 0.5450 & -1.1640 & -0.6562\\
			0.6860 & 0.7093 & -0.1849 & -0.7604\\
			1.0787 & -0.7736 & -5.2229 & -0.9054\\
			1.1337 & -1.7195 & -2.2162 & -1.9987
		\end{array}
		\right),
		W_3=\left(
		\begin{array}{cccc}
			-2.3787 & 0.5654 & 0.7364 & -1.0082\\
			5.3874 & -1.9893 & -0.2412 & 4.3542\\
			-0.9789 & 0.3109 & -0.5157 & -4.2987\\
			-0.2643 & 0.2798 & -1.4433 & 3.0057
		\end{array}
		\right),}\\
	&\scriptsize{
		W_4=\left(
		\begin{array}{cccc}
			-1.5792 & -0.3459 & 0.6400 & 0.3576\\
			1.0109 & -0.8013 & 0.7878 & 0.0865
		\end{array}
		\right),}
\end{array}
$$
$$
\begin{array}{ll}
	&\scriptsize{
		b_0=\left(
		\begin{array}{c}
			-1.1365\\
			0.1036\\
			-0.1027\\
			-0.0013
		\end{array}
		\right),
		b_1=\left(
		\begin{array}{c}
			0.8346\\
			0.4269\\
			-0.1510\\
			-0.1793
		\end{array}
		\right),
		b_2=\left(
		\begin{array}{c}
			-0.4180\\
			-0.7109\\
			-0.3882\\
			-0.9424
		\end{array}
		\right),
		b_3=\left(
		\begin{array}{c}
			0.5022\\
			0.4852\\
			-0.3491\\
			0.2299
		\end{array}
		\right),
		b_4=\left(
		\begin{array}{c}
			1.4647\\
			-0.7255
		\end{array}
		\right).}
\end{array}
$$
\vskip 0.2in
\bibliography{ReferenceFile}
\end{document}